\documentclass[runningheads]{llncs}

\newcommand{\orcidGP}	{\href{https://orcid.org/0000-0001-9394-6513}{\protect\includegraphics[scale=0.045]{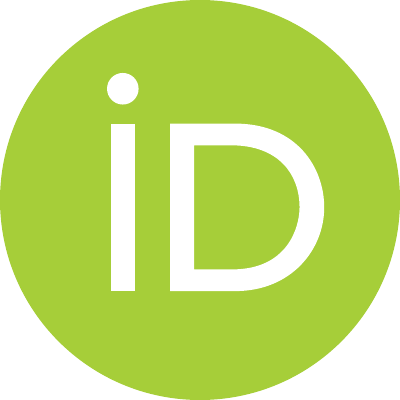}}}
\newcommand{\orcidSRM}	{\href{https://orcid.org/0000-0001-5656-6108}{\protect\includegraphics[scale=0.045]{orcid}}}

\newcommand{\orcidJB}	{\href{https://orcid.org/0000-0002-3979-400X}{\protect\includegraphics[scale=0.045]{orcid}}}
\usepackage{amssymb}
\usepackage{amsmath}
\usepackage{graphicx}
\usepackage{todonotes}
\usepackage{mathtools}
\usepackage{multirow}
\usepackage{subfigure}
\usepackage{wrapfig}
\usepackage{url}
\usepackage{subcaption}
\usepackage[inline]{enumitem}
\usepackage{listings}
\usepackage{amsfonts}
\usepackage{mathabx}
\usepackage{lscape}
\usepackage{enumitem}
\usepackage[T1]{fontenc}
\usepackage{verbatim}
\usepackage{wrapfig}
\usepackage{float}
\usepackage{comment}
\usepackage{hyperref}
\usepackage[noend]{algpseudocode}
\usepackage{booktabs}
\usepackage{makecell}
\usepackage{pifont}
\newcommand{\cmark}{\ding{51}}%
\newcommand{\xmark}{\ding{55}}
\usepackage[subtle]{savetrees}
\usepackage{algorithm}
\usepackage{tikz}
\usetikzlibrary{decorations.pathreplacing}

\newcommand{\pr}{^{\prime}}
\newcommand{\ea}{\operatorname{ea}}
\newcommand{\ma}{\operatorname{ma}}
\newcommand{\disc}{\operatorname{pd}}

\newcommand{\A}{\mathcal{A}}
\newcommand{\E}{\mathcal{E}}

\newcommand{\agg}{\operatorname{agg}}
\newcommand{\Lm}{\mathcal{L}_m}
\newcommand{\gen}{\operatorname{ntl}}
\newcommand{\NL}{\mathcal{M}}
\newcommand{\la}{\langle}
\newcommand{\ra}{\rangle}
\newcommand{\wtop}{w_{min\,max}}
\newcommand{\wm}{w_{max}}

\newcommand{\sea}{\ea_{\ma}}
\newcommand{\norm}[1]{\| #1 \|}

\newcommand{\AL}{A_L}
\newcommand{\AM}{A_M}
\newcommand{\Aa}{A_a}
\newcommand{\bag}{\mathbb{B}}
\newcommand{\Ltr}{L_{tmp}/\!\sim}
\newcommand{\Lar}{L_{a}/\!\sim}
\newcommand{\Ltc}{L_{tmp}^{class}}
\newcommand{\Lac}{L_{a}^{class}}

\newcommand{\mnomc}[2]{\genfrac{(}{)}{0pt}{2}{#2}{#1}}
\newcommand{\tloop}{\:\circlearrowleft\hspace{-0.1em}}

\newcommand{\dfc}{\sim_{df}}

\DeclareRobustCommand{\circled}[1]{\tikz[baseline=(char.base)]{\node[shape=circle,draw,inner sep=0pt,minimum size=3.2mm] (char) {#1};}}
\def\namedlabel#1#2{\begingroup
    #2%
    \def\@currentlabel{#2}
    \phantomsection\label{#1}\endgroup
}
\makeatother

\makeatletter
\def\namedlabel#1#2{\begingroup
    #2%
    \def\@currentlabel{#2}
    \phantomsection\label{#1}\endgroup
}
\makeatother
\definecolor{darkblue}{rgb}{.102,0.373,.706}
\newcommand{\janik}[1]{\textcolor{black}{#1}}
\newcommand{\rev}[1]{\textcolor{black}{#1}}

\newcommand{\ext}[1]{\textcolor{black}{#1}}

\newcommand{\srm}[1]{\textcolor{black}{#1}}

\begin{document}




\title{Synchronizing Process Model and Event Abstraction for Grounded Process Intelligence (Extended Version)}
\author{Janik-Vasily Benzin\inst{1}\orcidJB \and Gyunam Park\inst{2}\orcidGP \and Stefanie Rinderle-Ma\inst{1}\orcidSRM}

\titlerunning{Synchronizing Process Model and Event Abstraction}
\authorrunning{J.-V. Benzin et al.}
\institute{Technical University of Munich, TUM School of Computation, Information and Technology, Garching, Germany 
\email{\{janik.benzin,stefanie.rinderle-ma\}@tum.de} \and
Process~Mining~Group,~Fraunhofer~FIT,~Aachen,~Germany
\email{gyunam.park@fit.fraunhofer.de}}
\maketitle


\begin{abstract}
    
\srm{Model abstraction (MA) and event abstraction (EA) are means to reduce complexity of (discovered) models and event data. Imagine a process intelligence project that aims to analyze a model discovered from event data which is further abstracted, possibly multiple times, to reach optimality goals, e.g., reducing model size. So far, after discovering the model, there is no technique that enables the synchronized abstraction of the underlying event log. This results in loosing the grounding in the real-world behavior contained in the log and, in turn, restricts  analysis insights.  Hence, in this work, we provide the formal basis for synchronized model and event abstraction, i.e., we prove that abstracting a process model by MA and discovering a process model from an abstracted event log yields an equivalent process model. 
We prove the feasibility of our approach 
based on behavioral profile abstraction as non-order preserving MA technique, resulting in a novel EA technique.}


\end{abstract}









\section{Introduction}
\label{sec:intro}
Discovering a process model $M$ from an event log $L$ is a key step in analyzing the actual process behavior recorded by information systems \cite{leemans_robust_2022}. 
However, events are often logged at a low granularity level, leading to the discovery of complex and uninterpretable process models that do not match stakeholders' expectations. 
\emph{Event abstraction} (EA) techniques \cite{lim_framework_2024,van_zelst_event_2021} have been proposed to address this challenge by lifting the granularity of events \rev{to satisfy an \emph{abstraction goal} that formalizes stakeholder's expectations}.
While empirical studies have evaluated whether EA techniques can satisfy abstraction goals \rev{like model complexity reduction} \cite{van_houdt_empirical_2024}, existing EA techniques \rev{neither} provide formal guarantees on the reduction of model complexity \rev{nor on satisfying any other goal}. 
\rev{Hence, current EA techniques cannot be applied to find the optimal abstraction. 
Without optimality, downstream process  intelligence tasks such as process enhancement \cite{leemans_robust_2022}, business process simulation (BPS) \cite{lopez-pintado_discovery_2024}, and predictive process monitoring (PPM) \cite{ye_log_2025} are uncertain to meet the stakeholder's expectations.}  

\emph{Model abstraction} (MA) techniques \cite{smirnov_businessm_2012,mafazi_consistent_2015,senderovich_aggregate_2018}, by contrast, ensure \rev{satisfying an abstraction goal such as reducing model complexity through solving an optimization problem \cite{mafazi_consistent_2015,senderovich_aggregate_2018}. Specifically, the optimization objective is the abstraction goal and decision variables are the application sequence of abstraction operations and their parameters. } \srm{Hence, the question  arises \textbf{how the optimality guarantees of MA can be utilized for EA and, in turn, the knowledge of real-world process behavior stored in the logs can be exploited for process intelligence tasks on abstracted models, i.e., how to synchronize MA and EA (RQ). }}

\srm{For illustrating \textbf{RQ} and its effects on process intelligence, consider the scenario depicted in \autoref{fig:illustrative}: A bank runs trading processes for i) derivative and ii) fixed income products logged by two information systems. The bank wants to understand the common business process underlying the trading of i) and ii) through simulation and prediction tasks.}

\srm{The bank starts with discovering a process model $M$ (cf. \autoref{fig:illustrative}\circled{3}) from log $L$ merged from the two information systems (cf. \autoref{fig:illustrative}\circled{1}), i.e., $L\xrightarrow{pd}M$ using data-aware discovery, e.g., \cite{lopez-pintado_discovery_2024} after \cite{leemans_tech_2013}, as in this case also data objects are of interest. Assume now that in order to reduce the obvious complexity of $M$, behavioral profile abstraction (BPA) $ma_{bpa}$ \cite{smirnov_businessm_2012} abstracts $M$ into $M_a$ (cf. \autoref{fig:illustrative} \circled{5}) and subsequently, in order to generalize data objects to streamline the underlying processes, abstracts $M_a$ into $M_{aa}$ (cf. \autoref{fig:illustrative} \circled{7}). While $M_a$ and $M_{aa}$ fulfill the optimization goals, they have lost their grounding in an event log. Thus, the bank cannot proceed due to missing abstracted event logs $L_{a}$ (cf. \autoref{fig:illustrative} \circled{4}) and $L_{aa}$ (cf. \autoref{fig:illustrative} \circled{6}) that match the granularity of $M_a$ and $M_{aa}$ respectively and contain the actual observed behavior with rich event attributes. }

\begin{figure}[htb!]
    \centering
    \includegraphics[width=\linewidth]{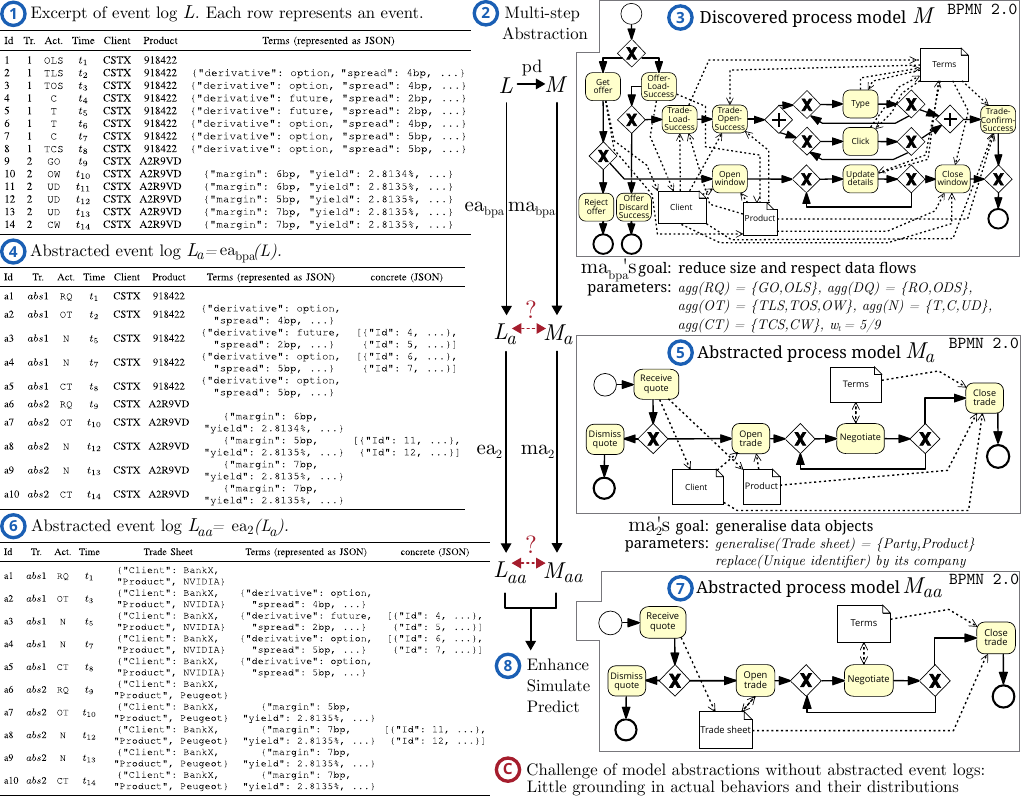}
   \caption{Process Intelligence in the Financial Domain (Example)}
    \label{fig:illustrative} 
\end{figure}

\srm{One solution would be to generate logs by playing out the abstracted models. However, this cannot reproduce the actual behavior stored in the log w.r.t., e.g., number of traces, frequency of paths, or data values. To overcome this information loss and to maintain flexibility of process intelligence tasks, we propose synchronizing MA with EA techniques to mirror the abstraction applied to $M$ on $L$. The core idea of synchronization is to prove that abstracting a model via $\operatorname{ma}$ and discovering a model from an log abstracted by $\operatorname{ea}$ yields equivalent models $M_a$ (annotated by question mark in \autoref{fig:illustrative}).}

 \srm{MA techniques can be distinguished into order-preserving and non-order-preserving techniques. In this work, we opt to study non-order-preserving techniques such as \cite{smirnov_businessm_2012} as they pose the more general, harder problem.} \rev{Moreover, MA techniques like \cite{smirnov_businessm_2012} enable both, unrestricted abstraction of control flow with data flow abstraction by clustering activities according to similar data flow or semantical control flow abstraction by clustering activities that are semantically-related according to domain-specific \emph{part-of} relations between activities \cite{smirnov_meronymy-based_2010}.} \srm{Overall, the contributions of this work include conditions under which MA and EA can be synchronized for non-order-preserving MA and an EA technique that maintains observed distributions in the abstracted log. As these contributions are conceptual and formal, the evaluation is the formal synchronization proof.}

Section \ref{sec:back} introduces background and related work. 
\autoref{sec:framework} establishes the theoretical foundation for synchronization by formalizing the problem and presenting our general approach. 
\autoref{sec:method} presents our concrete BPA-based synchronization method, including the adaptation of the BPA technique and the design of the corresponding synchronized EA technique.
\autoref{sec:theory} lays the theoretical foundations needed for proving synchronization correctness, and 
\autoref{sec:proof} proves synchronization correctness to establish the \rev{equivalence of MA's and EA's resulting process model.} \autoref{sec:illustrative} demonstrates the impact of synchronization on the illustrative example from the financial domain.
\autoref{sec:conclusion} concludes the paper.

\section{Background and Related Work}
\label{sec:back}

\srm{\textbf{Model representation, abstraction, and existing techniques:} As logical representation of process models,} we opt for process trees due to their \emph{block-structuredness}, often advocated in MA literature \cite{polyvyanyy_application_2009,smirnov_businessm_2012,mafazi_consistent_2015} and due to its favorable properties like \emph{soundness} \cite{leemans_tech_2013}. Let $\A$ be the
set of all possible activity names and the silent activity $\tau \not\in\A$. Then process trees are recursively defined by
\begin{itemize}
    \item $M = v$ for $v\in\A$ or $M = \tau$ are process trees (referred to as leaves), and
    \item $M = \oplus(M_1, \ldots, M_n)$ with $n$ process trees $M_1,\ldots, M_n$ and operators $\oplus \in \{\times,\rightarrow,\wedge, \circlearrowleft\}$ is a process tree referred to as a $\oplus$-node.
\end{itemize}

Let $\NL$ be the set of all models. Then, a \textbf{model abstraction (MA)} is a partial function $\ma: \NL \not\rightarrow \NL $ that maps a process model $M$ to another (abstracted) process model $M_a$ where it is assumed that the complexity of $M_a$ is reduced compared to $M$.
We call $\ma$ is \emph{applicable} to $M$ when $M \in dom(\ma)$.
\rev{The most common} complexity metric used in MA techniques is the size of the process model \cite{smirnov_business_2012,mafazi_consistent_2015} (cf. \autoref{tab:bpm_comparison}). \srm{In \autoref{fig:illustrative}, size $|M|=28<|M_{aa|}=12$ if we count all nodes in the models.} For process trees, the size of $|M|$ is defined as the sum of $\oplus$-nodes and leaves. We refer to the activities in a process tree $M$ with $A_M$, e.g., \rev{$A_{M_{aa}} = \{\texttt{RQ}, \texttt{OT}, \texttt{N}, \texttt{CT}\}$ in \autoref{fig:illustrative} \circled{7}}.

\ext{In \autoref{tab:bpm_comparison}, we report the 19 MA techniques that are not covered by the MA survey \cite{smirnov_businessm_2012} from 2012. In addition to the year of publication, we report the following nine properties for each technique. The \emph{abstraction goal} (goal in \autoref{tab:bpm_comparison}) specifies the target of the MA technique. The type specifies whether the abstraction operators are only \emph{order-preserving} (OP), only \emph{non-order-preserving} (NOP), or both. The \emph{model} reports the modeling language in which process models are represented. The \emph{process perspective} (Persp.) reports what perspective is represented in the model and, subsequently, abstracted by the technique. For the perspective, we abbreviate the \emph{control-flow} persp. by \texttt{C}, the \emph{data} persp. by \texttt{D}, and the \emph{organizational} persp. by \texttt{O}.}

\begin{table}
    \centering
    \resizebox{\linewidth}{!}{\begin{tabular}{lccccccccccc}
    \toprule
        Ref. & Year & Goal & Type & Model & Persp. & Op. & Abs. Obj. & Optim. & $\disc$ & $\ea$  \\
        \midrule
        \cite{ordoni_reduction_2023} & 2023 & Efficient verification  & OP & BPMN & \texttt{C,D} & \texttt{E} & SESE, data objects & 
\cmark & \xmark & \xmark\\
        \cite{volker_ontology-based_2023} & 2023 & Quick overview  & OP & BPMN & \texttt{C} & \texttt{A,E} & SESE & \xmark & \xmark & \xmark\\
         \makecell[l]{\cite{tsagkani_process_2022,tsalgatidou_rule-based_2016}\\\cite{tsagkani_abstracting_2015}} & 2022 & Quick overview & OP & BPMN Collab. & \texttt{C,D,O} & \texttt{A,E,G} & \makecell[c]{SESE, data objects,\\ lanes, messages} & \xmark & \xmark & \xmark \\
        \cite{angelastro_process_2020} & 2020 & Quick overview  & OP & WoMan \cite{ferilli_woman_2014} & \texttt{C} & \texttt{A} & SESE & \xmark & WIND \cite{ferilli_woman_2014} & \xmark \\
        \cite{wang_business_2019} & 2019 & Quick overview & OP & BPMN & \texttt{C} & \texttt{A} & SESE & \xmark & \xmark & \xmark \\
        \cite{perez-castillo_business_2019} & 2019 & Quick overview  & OP & BPMN & \texttt{C,D} & \texttt{A,G} & SESE, data objects &  \cmark & MARBLE \cite{perez-castillo_marble_2011} & \xmark  \\
        \cite{wang_business_2018} & 2018 & Quick overview & OP & BPMN & \texttt{C} & \texttt{A}  & SESE & \xmark & \xmark & \xmark \\
        \cite{senderovich_aggregate_2018,senderovich_p3-folder_2016} & 2018 & Efficient prediction & OP & GSPN & \texttt{C,P} & \texttt{A} & SESE & \cmark & IM \cite{leemans_discovering_2013} & \xmark \\
        \cite{de_san_pedro_log-based_2015} & 2015 & \makecell[c]{Quick overview,\\balanced discovery}  & NOP & PN & \texttt{C} & \texttt{E} & $P,F$ & \xmark & ILP \cite{van_der_werf_process_2008} & \xmark \\
        \cite{gaaloul_pql_2015} & 2015 & Configurable & (N)OP & BPMN & \texttt{C} & \texttt{A,E} & SESE, Flows & \xmark  & \xmark & \xmark \\
        \cite{mafazi_consistent_2015,mafazi_knowledge-based_2012} & 2015 & Configurable & OP & BPMN & \texttt{C,D,O} & \texttt{A,E,G} & \makecell{SESE, data objects,\\ resources} & \cmark & \xmark & \xmark  \\
        \cite{kopke_projections_2014} & 2014 & Privacy concerns  & OP & PT & \texttt{C} & \texttt{E} & SESE & \xmark & \xmark & \xmark \\
        \cite{fahland_simplifying_2013,fahland_simplifying_2011} & 2013 & Balanced discovery & OP & PN & \texttt{C} & \texttt{E} & $P,F$ & \xmark & All & \xmark \\
         \makecell[l]{\cite{kolb_flexible_2013,reichert_visualizing_2013}\\\cite{reichert_enabling_2012}} & 2013 & Custom views  & (N)OP & BPMN & \texttt{C,D,O} & \texttt{A,E,G} & \makecell{SESE, data objects,\\ activities, resources} & \xmark & \xmark & \xmark \\
        \cite{kolb_data_2013,kolb_updatable_2012} & 2013 & Change propagation  & (N)OP & BPMN &  \texttt{C,D,O} & \texttt{A,E,G} & \makecell{SESE, data objects,\\ activities, resources} & \xmark & \xmark & \xmark \\
        \makecell[l]{\cite{smirnov_businessm_2012,smirnov_fine-grained_2012}\\\cite{smirnov_business_2010,smirnov_meronymy-based_2010}}  & 2012 & Large repository  & NOP & BPMN & \texttt{C,D} & \texttt{A} & Activities & \xmark & \xmark & \xmark \\
        \cite{meyer_data_2012} & 2012 & Quick overview & OP & BPMN & \texttt{C,D} & \texttt{A} & SESE, data object & \xmark & \xmark & \xmark \\
        \cite{weber_refactoring_2011} & 2011 & \makecell{Large repository,\\quick overview}  & (N)OP & BPMN &  \texttt{C} & \texttt{A,E} & SESE, flows & \xmark & \xmark & \xmark \\
        \cite{xue_plain_2011} & 2011 & Quick overview & OP & CCS & \texttt{C} & \texttt{A,E} & SESE & \xmark & \xmark  & \xmark  \\
        \bottomrule
    \end{tabular}}
    \caption{\ext{Model Abstraction Techniques}}
    \label{tab:bpm_comparison}
\end{table}

\ext{Next, the abstraction \emph{operators} (Op.) show whether the technique applies \emph{elimination} (E), \emph{aggregation} (A), \emph{generalisation} (G), or a combination of the three. The \emph{abstraction objects} (Abs. Obj.) report what model elements in the domain of the abstraction operators, i.e., what model elements can be deleted, aggregated, or generalised. Obviously, the serialization of arbitrary BPMN process models into their \emph{refined process tree structure} \cite{vanhatalo_refined_2009} with single-entry single-exit (SESE) process fragments is prevalent among MA techniques.}

\ext{For the last three properties, we report whether the MA technique has the property or not. A MA technique applies optimization (Optim.) iff the abstraction goal is translated into an optimization objective and a solver is proposed that finds an optimal operator sequence and each operator's respective parameters that must be applied on a model to satisfy the abstraction goal. A MA techniques is formulated on discovered process models ($\disc$) iff it takes an event log $L$ as input, discovers a process model $M$ through process discovery technique $\disc$, and then abstracts $M$. Additionally, we report the process discovery technique that is considered in the MA technique. Lastly, a MA technique synchronizes its abstraction operators ($\operatorname{ea}$) iff for each operator a corresponding EA technique is defined that also abstracts the event log.}

\ext{Roughly half of MA techniques focus solely on the ``quick overview'' abstraction goal, i.e., aim to reduce the model size. Further selected goals in descending order are as follows. The goal ``large repository'' (2 times) is at par with ''configurable'' and ``balanced discovery''. While the former aims to manage a large repository of process models by only storing the fine-grained models and generating the coarse-grained models through MA, the latter two depend on the user through parametrization and the process discovery quality dimensions respectively. For example, \cite{gaaloul_pql_2015} proposes a process querying language that is capable of abstracting behavior before returning the result. The remaining five abstraction goals are each only once targeted. Interestingly, only two MA techniques target a goal, ``efficient verification'' and ``efficient prediction'', that considers process intelligence tasks beyond understanding and visualisation.}

\ext{The majority of 19 techniques propose order-preserving abstraction operators with 6/19 MA techniques considering operators that are non-order-preserving. MA techniques are commonly proposed for BPMN models with 7 exceptions: One technique is proposed for the declarative model language of the workflow management (WoMan) framework \cite{ferilli_woman_2014}, one technique is proposed for generalised stochastic Petri nets (GSPN), two techniques are proposed for Petri nets (PN), and one technique each for process trees (PT) and models in the \emph{calculus for communication systems} \cite{milner2009space}. Clearly, the control-flow perspective is always represented and the main target of any abstraction operator (cf. abstraction objects). 15/19 MA techniques}

In the following, we discuss the five MA techniques that consider an event log $L$ in more detail.
MA technique \cite{fahland_simplifying_2013} abstracts the discovered process model $M$ by filtering arcs in the \emph{unfolding} of $M$ and by applying three structural simplifications on the refolded $M$. 
Likewise, MA technique \cite{de_san_pedro_log-based_2015} abstract the discovered process model $M$ by either projecting $M$ into less complex model classes like \emph{series-parallel} Petri nets or by removing infrequently-enabled arcs detected through replay from $M$.
In both techniques, the mapping between activities in the model and events in the event log does not change, because both techniques neither abstract activities, nor events.
 MA technique \cite{senderovich_aggregate_2018} optimizes the application sequence of five order-preserving abstraction operators towards reducing the model size while controlling the information loss for efficiently predicting process performance. 
MA technique \cite{perez-castillo_business_2019} applies a sequence of abstraction operators on a discovered process model $M$; MA technique \cite{angelastro_process_2020} applies pattern mining to find order-preserving clusters of activities in $M$ to be abstracted. All three techniques \cite{senderovich_aggregate_2018,perez-castillo_business_2019,angelastro_process_2020} face the challenges of MA without abstracted event logs (cf. \autoref{sec:intro}): little grounding in actual behaviors and lack of flexibility for applying downstream process intelligence tasks. Besides, none of the three MA techniques considers non-order-preserving abstraction.

\textbf{Log representation and abstraction:} An event log $L$ is a multiset of traces, i.e., $L = [\sigma_{1},\ldots] \in \bag(\A^*)$. We denote the set of activities that occur in an event log with $\AL$ and write $e \in \sigma$ \rev{iff $e$ occurs in $\sigma$, i.e., $\exists i \in \{1, \ldots , |\sigma|\}: \sigma[i] = e$ with $\sigma[i]$ retrieving the $i$th event.} There are $10$ distinct activities (Act.) in $\AL$ for $L$ in \autoref{fig:illustrative} \circled{1}. \rev{We omit further rich event attributes like the ``Terms'' in $L$ in our trace conceptualization, because they are irrelevant for the proof of synchronization.}
\textbf{Event abstraction (EA)} is 
defined as a partial function $ \ea: \bag(\A^*) \not\rightarrow \bag(\A^*)$ such that the number of traces and events do not increase, i.e., $ |\ea(L)| \leq |L|$ and $\norm{\ea(L)} \leq \norm{L}$ with $\norm{L}=\sum_{\sigma\in L} |\sigma|$ \cite{van_zelst_event_2021}. 

\rev{We refer to \cite{van_zelst_event_2021,lim_framework_2024} for a detailed discussion and taxonomy on EA techniques and to \cite{van_houdt_empirical_2024} for an empirical evaluation of EA techniques published through 2024. EA technique \cite{ye_log_2025} aim to improve the accuracy of remaining time predictions that are learned from low-level event logs by applying three EA operators on the log. The three EA operators are designed to mirror the three MA operators sequence, self-loop, and choice. Yet, no theory is developed to prove the correctness of the designed EA operators. Moreover, the proposal is specific to remaining time prediction and does not show how to extend the EA technique with additional operators.}

\textbf{Process tree discovery and formal definitions:} The semantics of a process tree $\mathcal{L}(M)$ is the language represented by $M$ \cite{leemans_tech_2013}.
Given an event log $L$, a process discovery technique $\disc$ discovers a process tree $M$ that represents $L$, i.e., $\disc: \operatorname{B}(\A^*)  \rightarrow \NL$
with $\NL$ the set of all process trees.
For example, a process discovery technique, IM~\cite{leemans_robust_2022}, leverages the \emph{directly-follows graph} (DFG) $G(L) = (\AL \cup\{ \vartriangleright,\vartriangleleft\}, \mapsto_L)$ with $\vartriangleright, \vartriangleleft \:\not\in \AL$\footnote{$\vartriangleright$ is used to denote the start activity $\sigma = \la v,\ldots\ra$ as a directly-follows pair $(\vartriangleright, v) \in \mapsto_L$ and $\vartriangleleft$ analogously for the end activity of a trace.} and $\mapsto_L$ the \emph{directly-follows} relation to discover $M$. 
An event log  $L$  and a process tree  $M$  can be related based on the notion of \textsl{directly-follows completeness} \cite{leemans_robust_2022}, i.e., $L$ and $M$ are directly-follows complete (df-complete), denoted by $L \dfc M$, 
    iff the DFGs $G(L)$ and $G(M) = G(\mathcal{L}(M))$ are equal: $G(L) = G(M)$.
Df-completeness captures the behavior in $L$ and $M$ as equivalent to the abstract representation of a DFG. 
Df-completeness is a condition for the Inductive Miner (IM) to \emph{rediscover} a process tree $M$ from $L$ that is \emph{isomorphic} to $M\pr$ that was executed for recording the event log $L$ and, as such, is integral to the EA techniques presented in \autoref{sec:sea} and \autoref{sec:proof}. 

Two process trees $M_1, M_2$ are isomorphic, formally $M_1  \cong M_2$, iff they are syntactically equivalent up to reordering of children for $\wedge $- and $\times$-nodes and the non-first children of $\circlearrowleft\,$-nodes. 
A process tree $M$ is \textsl{isomorphic rediscoverable} by $\disc$ from event log $L$ with $L \subseteq \mathcal{L}(M)$ iff $\disc$ discovers a process tree $ M\pr = \disc(L)$ that is isomorphic to $M$ \cite{leemans_tech_2013}. Isomorphic rediscoverability has been proven for the IM 
through assuming a restriction $Q(M)$ that must hold for process tree $M$ and a restriction $R(L,M)$ that must hold for $L$ and $M$.
$Q(M)$ requires a process tree without silent activities $\tau$, \emph{duplicate} activities, and joint start and end activities of a $\circlearrowleft\,$-nodes first child and $R(L,M)$ requires df-completeness \cite{leemans_tech_2013}.  


\section{Synchronization Framework}\label{sec:framework}

In this section, we establish the theoretical foundation for synchronizing MA and EA techniques. We first formalize the synchronization problem and requirements, then present our general synchronization approach.





\noindent \textbf{Synchronization problem and requirements.} For synchronization to be possible, we must be able to design an EA technique $\sea$ that transforms the event log such that process discovery yields an abstracted model isomorphic to what we would obtain through direct model abstraction.

\begin{definition}[Synchronizability]
\label{def:sync}
Let $L$ be an event log and $M = \disc(L)$ the discovered process tree such that model abstraction $\ma$ is applicable to $M$, i.e., $M_a = \ma(M)$. $L$, $\ma$, and $\disc$ are synchronizable iff there exists event abstraction $\sea$ \rev{s.t.} $\disc(\sea(L)) \cong M_a$.
\end{definition}


\noindent \textbf{Synchronization approach.} We follow a two-step approach for synchronizing MA and EA: first discovering a complex process tree $M = \disc(L)$ followed by applying $\ma$ to yield $M_a = \ma(M)$, and abstracting $L_a = \sea(L)$ followed by discovering $M_a\pr = \disc(L_a)$ should result in isomorphic abstract process trees, i.e., $M_a \cong M_a\pr$. 
Therefore, our approach requires two key components:
\begin{itemize}
    \item MA Technique Adaptation: Not all MA techniques are immediately suitable for synchronization. 
    We need to ensure the MA technique is well-defined and provides sufficient structure for designing a corresponding EA technique.
    \item Synchronized EA Design: The EA technique must be designed to mirror the abstractions applied by the MA technique while preserving the behavioral relationships needed for correct process discovery.
\end{itemize}

We demonstrate this approach using the behavioral profile abstraction (BPA) technique \cite{smirnov_businessm_2012} for the following reasons. First, synchronizing BPA constitutes a significant challenge because it allows abstracting arbitrary sets of activities \rev{to allow abstractions wrt. the data flow (cf. \autoref{fig:illustrative}).} Hence, the corresponding EA technique must abstract the event log while guaranteeing the correct order of activities in the abstracted event log. 
Second, BPA aims to reduce model size, which aligns well with the size characteristics of event logs.

We select the Inductive Miner (IM) including fall-throughs for process discovery \cite{leemans_robust_2022} to balance practical relevance with proof complexity, as isomorphic rediscoverability is already established for IM.
\rev{Because our synchronization approach requires the novel EA technique $\ea_{bpa}$ to result in abstracted event logs $L_a$ for which the IM discovers $M_a$ to maintain the relation between log and model, $L_a$ enables further process intelligence tasks that are grounded in the real-world behavior of $L_a$ and that were not possible before (cf. \autoref{sec:intro}).}

\section{BPA-Based Synchronization Method}\label{sec:method}
In this section, we present our concrete synchronization method based on BPA. 
We first adapt the BPA technique to ensure it meets our synchronization requirements (\autoref{sec:bpa}), then design the corresponding synchronized event abstraction technique (\autoref{sec:sea}).

\subsection{Adapting BPA for Synchronization}
\label{sec:bpa}

\ext{This section introduces the non-order-preserving BPA technique $\ma_{bpa}$ and adapts it for synchronization. To illustrate the three steps of $\ma_{bpa}$, we introduce a running example. In addition to the illustrative example in \autoref{fig:illustrative}, we present the running example that is more detailed to demonstrate the full behavior of both $\ma_{bpa}$ in the following, our synchronized EA technique $\ea_{bpa}$ in \autoref{sec:sea}, and our algorithm to generate \emph{minimal} df-complete event logs in \autoref{sec:min}. Hence, the running example is used to demonstrate the algorithms step-by-step, whereas the illustrative example motivates our approach in \autoref{sec:intro} and is revisited in \autoref{sec:illustrative}.}

\begin{figure}[htb!]
    \centering
    \includegraphics[width=\linewidth]{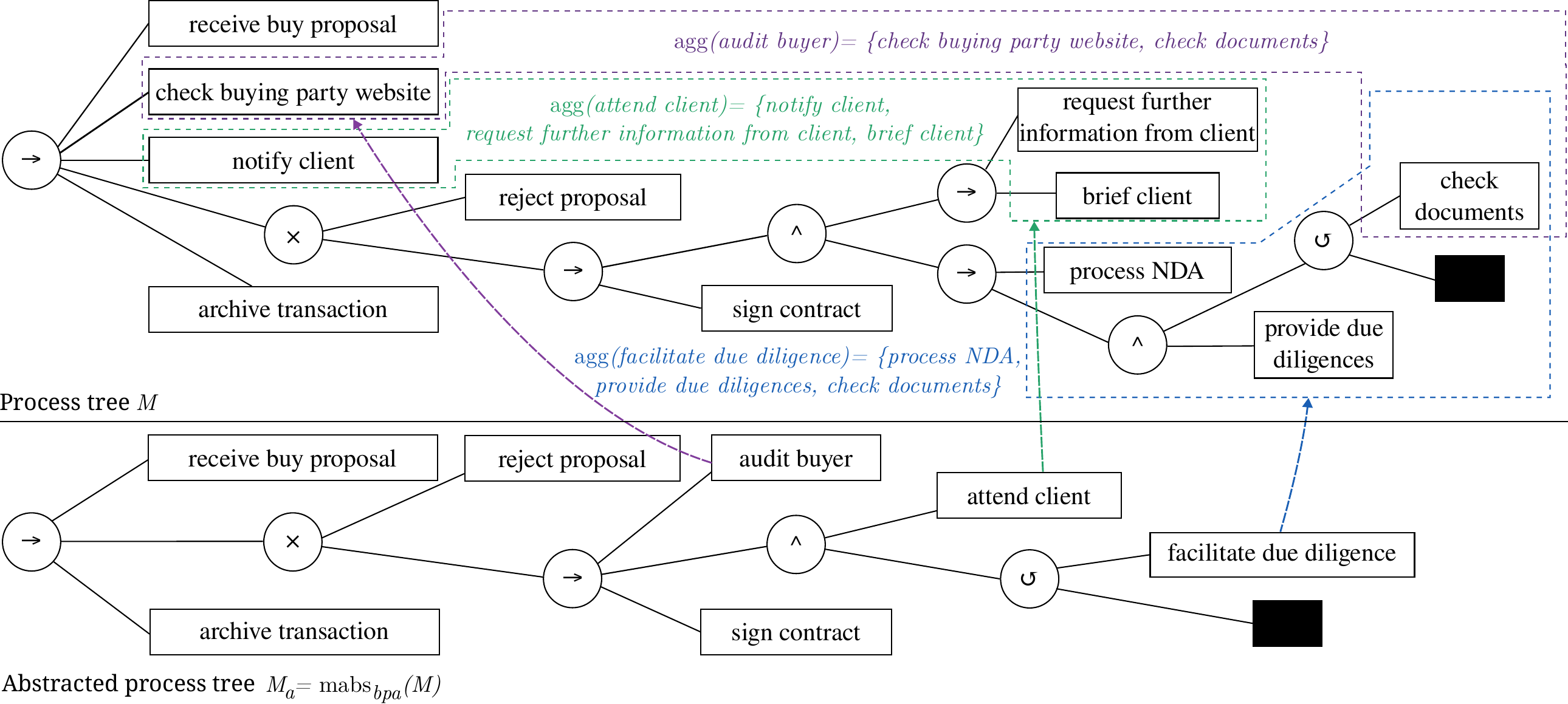}
    \caption{Running example: Process tree $M$ and abstracted process tree $\ma_{bpa}(M)$ \cite{smirnov_businessm_2012} for threshold $w_t =0.5$.}
    \label{fig:ex} 
\end{figure}
\vspace{-0.1em}

\begin{wraptable}[8]{R}{0.5\linewidth}
\vspace{-2em}
\resizebox{\linewidth}{!}{
\begin{tabular}{lc}   
    \toprule
    ref & trace  \\
    \midrule
    $\sigma_{ex1}$  & $\la \texttt{RBP}, \texttt{CBW}, \texttt{NC}, \texttt{RP}, \texttt{AT}\ra$ \\ 
    $\sigma_{ex2}$ & $\la \texttt{RBP}, \texttt{CBW}, \texttt{NC}, \texttt{RFI}, \texttt{BC}, \texttt{PN}, \texttt{CD},\texttt{CD}, \texttt{PDD}, \texttt{SC}, \texttt{AT}\ra$\\ 
    $\sigma_{ex3}$ &  $\la \texttt{RBP}, \texttt{CBW}, \texttt{NC}, \texttt{RFI}, \texttt{PN}, \texttt{BC}, \texttt{CD},\texttt{CD}, \texttt{PDD}, \texttt{SC}, \texttt{AT}\ra$\\ 
    $\sigma_{ex4}$ &  $\la \texttt{RBP}, \texttt{CBW}, \texttt{NC}, \texttt{PN}, \texttt{RFI}, \texttt{BC}, \texttt{CD},\texttt{CD}, \texttt{PDD}, \texttt{SC}, \texttt{AT}\ra$\\ 
    \bottomrule
\end{tabular}}
 \caption{Excerpt of event log $L_{2}$. }
  \label{tab:log}
\end{wraptable}

\noindent\textbf{Running example}: Figure \ref{fig:ex} shows \emph{process tree} $M$ which describes a transaction process by a mergers \& acquisitions advisor that sells companies for its clients. The advisor receives buy proposals, checks their websites, and notifies the client about the new proposal. If the proposal is convincing, further information for the buyer is requested from the client and the client is briefed while the advisor processes the non-disclosure agreement (NDA), followed by providing the buyer with confidential documents (due diligence) and checking the documents received from the buyer (multiple times). In the end, an acquisition contract is signed and the proposal is archived. If the proposal is not convincing, it is rejected. Lastly, Table \ref{tab:log} shows four traces $\sigma_{ex1}, \ldots$ of an event log $L_2$ that is recorded from executing $M$ (activity names are abbreviated by their acronym).

\ext{BPA aims to enable unrestricted abstraction of concrete activities into abstract activities. Consequently, we can cluster activities according to their data flow (cf. \autoref{fig:illustrative} \circled{3}) and apply $\ma_{bpa}$ to abstract accordingly.}  
At the core of $\ma_{bpa}$ lies the \emph{behavioral profile} of a process tree $M$. It is equivalent to the \emph{footprint} \cite{van_der_aalst_workflow_2004} of a process model and constitutes an abstract representation of the behavior allowed by the model similar to the DFG (cf. \autoref{sec:back}), but more coarse-grained. Hence, a behavioral profile contains less detailed information on the behavior of the model than the DFG, e.g., a loop $\tloop(a,b)$ leaves a distinct graph pattern in the DFG, but is indistinguishable from $\wedge(a,b)$ in a behavioral profile. The behavioral profile is defined as follows:

\begin{definition}[Behavioral profile, adapted from \cite{smirnov_businessm_2012}]
\label{def:bp}
    Let $M$ be a process tree, $A = \Sigma(M)$ its activities, and $\mathcal{L}(M)$ its language. Let $\succ\; \subseteq A\times A$ be the weak-order relation that contains all activity pairs $(x,y) \in \;\succ$ for which there exists a trace $\sigma = \langle \alpha_1, \ldots, \alpha_n\rangle \in \mathcal{L}(M)$ with $j \in \{1, \ldots, m - 1\}$ and $j < k\leq m$ such that $\alpha_j=x$ and $\alpha_k = y$. Given the weak-order relation, an activity pair $(x,y)$ is either:
    \begin{itemize}
        \item in a strict order relation $ \rightsquigarrow_{M}$: $ x \succ y $ and $y \not\succ x$,
        \item in an choice order relation $+_{M}$: $x \not\succ y$ and $y \not\succ x$,
        \item or in a parallel order relation $\parallel_{M}$: $x \succ y$ and $y \succ x$.
    \end{itemize}
    The set of all three relations $ \rightsquigarrow_{M}, +_{M},$ and $\parallel_{M}$ is the behavioral profile $p_{M}$ of $M$. The set of all behavioral profiles over activities $A$ is denoted by $\mathcal{BP}_{A}$.
\end{definition}

As aforementioned, we acronym the activity names. For example, we have $\texttt{RBP} \rightsquigarrow_M \texttt{RP}$, $\texttt{RP} +_M \texttt{SC}$, and $\texttt{CD} \parallel_M \texttt{CD}$ in the behavioral profile $p_M$ of the concrete process tree $M$ in \autoref{fig:ex}. Given the behavioral profile notion, we can introduce $\ma_{bpa}(M) = M_a$ as three subsequent steps: 
\begin{enumerate}
    \item[S1] The behavioral profile is computed: $\operatorname{pr}(M) = p_M \in \mathcal{BP}_{A}$.
    \item[S2] Given a behavioral profile, the abstract behavioral profile $p_{M_a}$ is derived from $p_{M}$ by a parametrized function: $\operatorname{dv}_{\agg, w_t}(p_M) = p_{M_a}$. The first parameter is a function $\agg: A_a \rightarrow 2^A$ with $A_a$ the activities of $M_a$ without the silent activity and $A = \Sigma(M)$. $\agg$ specifies which abstract activities correspond to which sets of concrete activities. The second parameter $0 < w_t \leq 1$ controls what ordering relation frequencies are selected for $p_{M_a}$ from $p_M$. 
    \item[S3] Given an abstract behavioral profile, an abstracted process tree $M_a$ is synthesized whose behavioral profile equals $p_{M_a}$, i.e., $\operatorname{sy}(p_{M_a}) = M_a $. To uniquely construct a process tree $M_a$ from profile $p_{M_a}$, the profile is encoded as a graph $G(p_{M_a})$ and the graph's unique \emph{modular decomposition tree} $MDT(G)$ \cite{mcconnell_linear-time_2005} is computed. If each \emph{module} $m$ in $MDT(G)$ is either \emph{linear}, \emph{AND}-, or \emph{XOR}-complete \cite{smirnov_businessm_2012}, then process tree $M_a$ is constructed by adding a tree node for each module. As a module can be \emph{primitive}, i.e., contains ``conflicting'' ordering relations, not all profiles $p_{M_a}$ have a corresponding process tree $M_a$ such that this step may fail: $\operatorname{sy}(p_{M_a}) = \bot $.
\end{enumerate} 


To illustrate $\ma_{bpa}$, we refer to \autoref{fig:ex}. The process tree $M$, the abstracted process tree $M_a$, and the parameter $\agg$ are depicted. First, the behavioral profile $p_M$ of $M$ is computed (S1).
Given $p_M$, the second step $\operatorname{dv}_{\agg, w_t}(p_M)$ is computed (S2). The parameter $\agg$ is denoted in \autoref{fig:ex} by three different colors.
For instance, $\agg(\texttt{AB}) = \{\texttt{CBW}, \texttt{CD}\}$. For presentation purposes, the mappings $\agg(y) = \{y\} $ for $y \in \{\texttt{RBP}, \texttt{RP}, \texttt{SC}, \texttt{AP}\}$ are not visualized in \autoref{fig:ex}. Note that at this stage, the order of abstract activities in the abstracted process tree $M_a$ is unknown and cannot be derived intuitively, because their mappings of $\agg$ may overlap (e.g., \texttt{AB} and \texttt{FDD}) or the order of their concrete activities is in conflict (e.g., \texttt{NC} and \texttt{BC} in $\agg(\texttt{AC})$ vs. \texttt{CD} in $\agg(\texttt{FDD})$ are in strict and interleaving order respectively). $\ma_{bpa}$ computes the ordering relations between two abstract activities $x, y \in M_a$ by selecting the most \emph{restrictive} ordering relation among those that occur relatively more frequent than or equally frequent to the threshold $w_t$. To that end, $\operatorname{dv}_{\agg, w_t}$ applies \autoref{alg:bpa} to each abstract activity pair $x$ and $y$. 

\begin{algorithm}[t]
    \caption{Derivation of an ordering relation (adapted from \cite{smirnov_businessm_2012})}
    \label{alg:bpa}
    \begin{scriptsize}

    \begin{algorithmic}[1]
    \State $\mathbf{deriveOrderingRelation}_{\agg, w_t}(\mathbf{Activity} \;x, \mathbf{Activity}  \;y)$
        \State $w(x \succ_{M_a} y) = |\{\forall (v, u) \in \agg(x) \times \agg(y) : v \rightsquigarrow_{M} u \lor v \parallel_{M} u\}|$
        \State $w(y \succ_{M_a} x) = |\{\forall (v, u) \in \agg(x) \times \agg(y) : v \rightsquigarrow_{M}^{-1} u \lor v \parallel_{M} u\}|$
        \State $w(x \not\succ_{M_a} y) = |\{\forall (v, u) \in \agg(x) \times \agg(y) : v \rightsquigarrow_{M}^{-1} u \lor v +_{M} u\}|$
        \State $w(y \not\succ_{M_a} x) = |\{\forall (v, u) \in \agg(x) \times \agg(y) : v \rightsquigarrow_{M} u \lor v +_{M} u\}|$
        \State $w_{prod} = |\agg(x)| \cdot |\agg(y)|$
        \State $w(x +_{M_a} y) = \min(w(x \not\succ_{M_a} y), w(y \not\succ_{M_a} x)) \cdot \frac{1}{w_{prod}}$
        \State $w(x \rightsquigarrow_{M_a} y) = \min(w(x \succ_{M_a} y), w(y \not\succ_{M_a} x)) \cdot \frac{1}{w_{prod}}$
        \State $w(x \rightsquigarrow_{M_a}^{-1} y) = \min(w(y \succ_{M_a} x), w(x \not\succ_{M_a} y)) \cdot \frac{1}{w_{prod}}$
        \State $w(x \parallel_{M_a} y) = \min(w(x \succ_{M_a} y), w(y \succ_{M_a} x)) \cdot \frac{1}{w_{prod}}$
        \If{$w(x +_{M_a} y) \geq w_t$}
            \State \Return $x +_{M_a} y$
        \ElsIf{$w(x \rightsquigarrow_{M_a} y) \geq w_t$}
            \If{$w(x \rightsquigarrow_{M_a}^{-1} y) > w(x \rightsquigarrow_{M_a} y)$}
                \State \Return $x \rightsquigarrow_{M_a}^{-1} y$
            \Else
                \State \Return $x \rightsquigarrow_{M_a} y$
            \EndIf
        \ElsIf{$w(x \rightsquigarrow_{M_a}^{-1} y) \geq w_t$}
            \State \Return $x \rightsquigarrow_{M_a}^{-1} y$
        \ElsIf{$w(x \parallel_{M_a} y) \geq w_t$}
            \State \Return $x \parallel_{M_a} y$
        \Else
            \State \Return $x \parallel_{M_a} y$
        \EndIf
    \end{algorithmic}
\end{scriptsize}
    \end{algorithm}

\autoref{alg:bpa} consists of three blocks: Counting frequencies of weak order relations between the respective concrete activities (line 2-5), deriving relative frequencies for ordering relations from weak order relations (line 6-10), and selecting the most restrictive ordering relation ($+$ > $\rightsquigarrow^{-1}$ > $\rightsquigarrow$ > $\parallel$) that is equal to or greater than threshold $w_t$. 
For example, \autoref{alg:bpa} applied to \texttt{AB} and \texttt{AC} for $w_t = 0.5$ computes the relative weak order frequency $w(\texttt{AB} \succ_{M_a} \texttt{AC}) =5/6$, because $\texttt{CBW} \succ \texttt{NC}, \texttt{CBW} \succ \texttt{RFI}, \texttt{CBW} \succ \texttt{BC}, \texttt{CBW} \succ \texttt{NC}, \texttt{CD} \succ \texttt{RFI}$, and $ \texttt{CD} \succ \texttt{BC}$. Analogously, we have $w(\texttt{AC} \succ_{M_a} \texttt{AB}) =3/6$, $w(\texttt{AB} \not\succ_{M_a} \texttt{AC}) =1/6$, and $w(\texttt{AC} \not\succ_{M_a} \texttt{AB}) = 3/6$. The weak order frequencies are transformed to order relation frequencies by taking the minimum of the respective two weak order relations. Thus, $w(\texttt{AB} +_{M_a} \texttt{AC}) = 1/6$, because $1/6$ is the minimum of $w(\texttt{AB} \not\succ_{M_a} \texttt{AC}) $ and $ w(\texttt{AC} \not\succ_{M_a} \texttt{AB})$. Analogously, we have $w(\texttt{AB} \rightsquigarrow_{M_a} \texttt{AC}) = 3/6$, $w(\texttt{AB} \rightsquigarrow_{M_a}^{-1} \texttt{AC}) = 1/6$, and $w(\texttt{AB} \parallel_{M_a} \texttt{AC}) = 3/6$ such that $\texttt{AB} \rightsquigarrow_{M_a} \texttt{AC}$ is the most restrictive ordering relations whose relative frequency is equal to $w_t$. Overall, the result is $\texttt{AB} \rightsquigarrow_{M_a} \texttt{AC}$.




As the third step (S3), $\ma_{bpa}$ attempts synthesizing the abstracted process tree $M_a$ from the abstract behavioral profile $p_{M_a} $ (cf. Algorithm 3.2 in \cite{smirnov_businessm_2012}). To construct the different nodes of the process tree $M_a$, an \emph{order relations graph} $G(p_{M_a}) = (V,E)$ for a given behavioral profile $p_{M_a}$ is constructed. The nodes are the activities $V = A_a$. Edges correspond to the strict order relation and the choice relation without the identity, i.e., $E = \; \rightsquigarrow_{M_a} \cup +_{M_a} \setminus \;id_{A_a}$. For example, \autoref{fig:mdt} (a) depicts the order relations graph $G(p_{M_a})$ for the abstracted behavioral profile $p_{M_a}$ that is derived in step 2 for the running example in \autoref{fig:ex}.

\begin{figure}[ht!]
    \centering
    \includegraphics[width=0.85\linewidth]{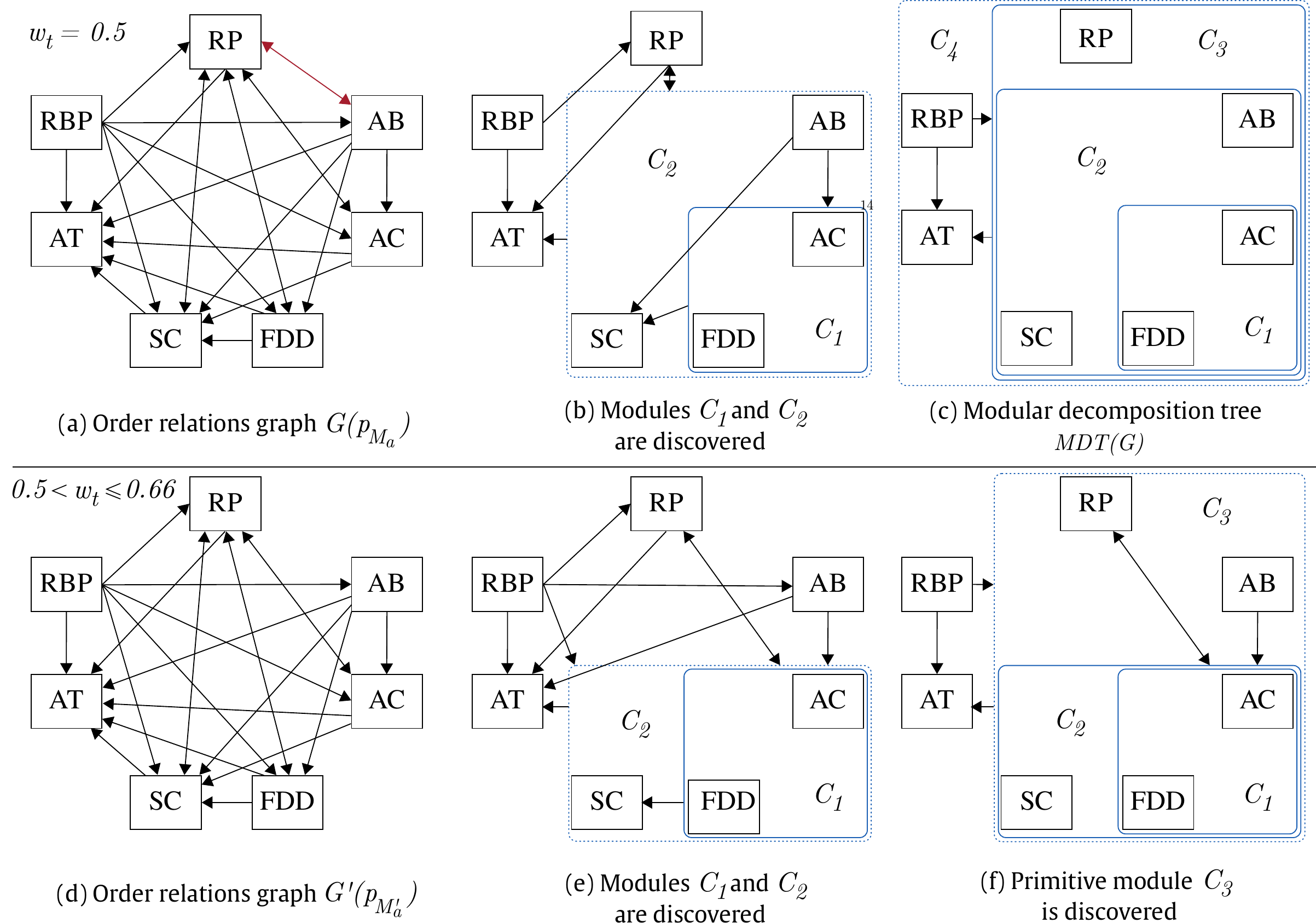}
    \caption{Modular decomposition of two order relation graphs $G(p_{M_a})$ and $G(p_{M_a\pr})$ that are derived by $\ma_{bpa}$ from the running example $M$ for two different parameters $w_t$. For $w_t = 0.5$ (a-c), AND-complete module $C_1$, linear module $C_2$, XOR-complete module $C_3$ and linear module $C_4$ are discovered. For $0.5 < w_t \leq 0.66$ (d-e), AND-complete module $C_1$, linear module $C_2$, and primitive module $C_3$ are discovered.}
    \label{fig:mdt} 
  \end{figure}

To derive a unique tree structure from $G(p_{M_a})$, the modular decomposition tree $MDT(G)$ \cite{mcconnell_linear-time_2005} is computed. The tree contains a hierarchy of \emph{non-overlapping}\footnote{Two modules \emph{overlap} iff they intersect and neither is a subset of the other.} modules $C \subseteq V$ that have uniform ordering relations with activities $V \:\setminus\: M$, i.e., they ``agree'' on their ordering relations to other activities. Additionally, modules are classified through the ordering relations between their activities $x \in C$: AND-complete and XOR-complete modules have only activities that are in interleaving order (i.e., they are not connected in $G(p_{M_a})$) and in choice order (i.e., they are completely connected) respectively, while linear modules have only activities that can be linearly ordered such that their edges do not violate the direction of the linear order. Any other module is primitive.

As depicted in \autoref{fig:mdt} (a-c), the modular decomposition of $G(p_{M_a})$ discovers four modules that each correspond to a node of $M_a$ (cf. \autoref{fig:ex}): Linear module $C_4$ corresponds to the $\rightarrow(\ldots)$ root node in $M_a$, XOR-complete module $C_3$ corresponds to the $\times(\ldots)$ node in $M_a$, linear module $C_2$ to the second $\rightarrow(\ldots)$ node, and AND-complete module $C_1$ to the $\wedge(\ldots)$ node. Because the modular decomposition $MDT(G)$ (c) does not contain any primitive module, the abstracted process tree $M_a$ with behavioral profile $p_{M_a}$ can be synthesized. Additionally, step 3 includes a special case for activities $x \in A_a$ that are in parallel order with themselves: $x \parallel x\in p_{M_a}$. $\operatorname{sy}$ constructs a self-loop node $\tloop(x, \tau)$, e.g., the self-loop of activity \texttt{FDD} in \autoref{fig:ex}.


\rev{In general, parameter $\agg$ is assumed to be set (i.e., computed by clustering cf. \autoref{sec:intro}), as the actual value has no impact on our results.} We explicitly add three restrictions (3-5) on $ma_{bpa}$ to guarantee that $ma_{bpa}$ always satisfies its abstraction goal (cf. \autoref{fig:illustrative} \circled{3}):

\begin{definition}[Behavioral Profile Abstraction (adapted from \cite{smirnov_businessm_2012})]
\label{def:bpa}
    The behavioral profile abstraction $\ma_{bpa}(M) = M_a$, $M_a = \operatorname{sy}(p_{M_a})$, \\$p_{M_a} = \operatorname{dv}_{\agg, w_t}(p_M)$, $p_M = \operatorname{pr}(M)$ is applicable to $M$ iff\footnote{We denote $A_M$ by $A$ and $A_{M_a}$ by $M_a$, because the context is clear.}: 
    \begin{enumerate}
    \item $M$ has no duplicate activities,
    \item the modular decomposition tree $MDT(G)$ of the abstract behavioral profile's graph $G(p_{M_a})$ contains no primitive module,
    \item 
    ($A_{new} :=\Aa \setminus A  \neq \emptyset$) $\wedge$ 
    ($A_c := \Aa \setminus A_{new}  \subseteq A $) $\wedge$  ($A_{new} \cap A_c=\emptyset $),
    \item 
 ($\forall x \in \A_{new}$: $|\agg(x)|>1 )\wedge $
 ($|\bigcup_{y\in A_{new}} \agg(y)| > |A_{new}| + 1$) $\wedge$
 ($\forall y \in A_c$: $\agg(y) = \{y\}$),
 \item $w_t$ is restricted to $0 < w_t \leq \wtop $ with $\wtop= \min_{x,y\in \Aa}\wm(x,y) $ and 
    \begin{align*}
    \vspace{-0.5em}
        \wm(x,y) = \max\bigl(w(x +_{M_a} y), w(x \rightsquigarrow_{M_a} y), \\w(x \rightsquigarrow_{M_a}^{-1} y), w(x \parallel_{M_a} y)\bigr). 
    \end{align*}
 \end{enumerate}
\end{definition}

Conditions (1-2) are required for the applicability of $\ma_{bpa}$ by the original proposal in \cite{smirnov_businessm_2012}. Conditions (3-5) are added to guarantee that the resulting process tree $M_a$ is smaller and to prohibit renaming of activities during abstraction.

\noindent \textbf{Conditions (1-2)}: The original $\ma_{bpa}$ \cite{smirnov_businessm_2012} additionally required both the process tree $M$ to be of the form $M = \:\rightarrow(s,M\pr,e)$ for start and end activities $s$ and $e$ and the abstracted process tree $M_a$ to be of the form $M_a = \:\rightarrow(s_a,M_a\pr,e_a)$. In $M_a$, the activities $s_a$ and $e_a$ are either equal to their concrete counterparts $s$ and $e$ or are added in step 3 as artificial start and end activities, if $M_a$ would otherwise not have a start and end activity. Both restrictions are not required for our purpose and, in the case of adding artificial start and end activities, results in a non-synchronizable model abstraction. Moreover, the abstracted process tree $M_a$ may not be smaller than the process tree $M$, i.e., $\ma_{bpa}$ may not be a model abstraction anymore. Consequently, we adapted $\ma_{bpa}$ by removing both restrictions.

\noindent \textbf{Conditions (3-5)}: Restrictions (3), (4), and (5) guarantee that the resulting process tree $M_a$ is smaller and prohibit renaming of activities during abstraction, e.g., $|M| = |\ma_{bpa}(M)|$ for any $\agg$ that only renames activities in \autoref{fig:ex}. BPA $\ma_{bpa}$ with $A_{new} = \{\texttt{AB}, \texttt{AC}, \texttt{FCC} \}$, $w_t = \wtop = 0.5$, and $\agg$ as depicted in \autoref{fig:ex} is applicable to $M$. To illustrate restriction (5), consider the running example. By setting the parameter $w_t$ to a value between $0.5$ and $0.66$, the order relations graph misses the choice edge between \texttt{RP} and \texttt{AB} (denoted in red in \autoref{fig:mdt} (a)), because the default case (line 22) of \autoref{alg:bpa} is reached for these two activities. Hence, order relations graph $G(p_{M_a\pr})$ is derived in step 2 as depicted in \autoref{fig:mdt} (d). Due to the missing edge, a primitive module $C_3$ is discovered by the modular decomposition (cf. \autoref{fig:mdt} f). Consequently, no abstracted process tree $M_a$ exists that has the same behavioral profile $p_{M_a\pr}$. %
To avoid the default case (line 22) and subsequent discovery of primitive modules, we restrict $w_t$ in restriction (5). 

Importantly, these restrictions have not been stated for $\ma_{bpa}$ in \cite{smirnov_businessm_2012}. Thus, adding (3), (4), and (5) constitutes an adaptation of $\ma_{bpa}$. Our adapted $\ma_{bpa}$ is well-defined (cf. \autoref{sec:back}). 
Next, we present our design for $\sea$.

\subsection{Synchronized Event Abstraction Design}
\label{sec:sea}

\janik{To ensure synchronization for our novel EA technique $\ea_{bpa}$, we design $\ea_{bpa}$ to transform $L$ into an abstracted event log $L_a$ such that $L_a$ and $M_a$ are df-complete. We aim for a df-complete $L_a$, because df-completeness is required for IM's isomorphic rediscoverability (cf. \autoref{sec:back}). 
$\ea_{bpa}$ is composed of two steps.}


\noindent \textbf{Preliminarily Abstracting The Event Log.} First, $\ea^1_{bpa}$ constructs a preliminary abstracted event log $L_{tmp}$ by abstracting occurrences of concrete events $e \in A_{\ma}$ (i.e., $e$ is abstracted by $\ma_{bpa}$) into new abstract events $x$ trace by trace. 
Next, $\ea^1_{bpa}$ deletes abstract activities $x$ that are in choice relation to another abstract activity $y$ from traces $\sigma_{abs} \in L_{tmp}$ in which both $x$ and $y$ occur. We illustrate $\ea_{bpa}^1$ with the running example (cf. \autoref{fig:ex}). Let $L$ be an event log such that IM discovers $M$ as depicted. For simplicity, we assume that $L$ is a \emph{minimal} df-complete event log (cf. \autoref{sec:min}).

\begin{algorithm}
\scriptsize
\caption{First step to synchronize $\ma_{bpa}$: $\ea^1_{bpa}$ }
\begin{algorithmic}[1]
\Require: Event log $L$, process discovery technique $\disc_{IM}$, MA technique $\ma_{bpa}$ \rev{with corresponding $p_{M_a}, A_a, A_{new},$ and $\agg$ (cf. \autoref{def:bpa})}
\Ensure: Preliminary abstracted event log $L_{tmp}$, abstracted process tree $M_a$
\State $L_{tmp} \gets \{\}, \;M \gets \disc_{IM}(L), \;M_a \gets \ma_{bpa}(M), $ 
\ForAll{traces $\sigma \in L$}
\State $\sigma_{abs} \gets \la\ra$
\State $A_{\ma} \gets \{e \in \sigma\mid \exists x \in A_{new}:  e \in \agg(x)\}$
\State $A_{\neg\ma} \gets \{e \in \sigma \} \,\setminus\, A_{\ma}$
\ForAll{$e \in \sigma$ in the order of $\sigma$}
\If{$e \in A_{\ma}$}
\ForAll{$x\in A_{new}$ with $e \in \agg(x) $ first appearing in $\sigma$}
\If{$v \in A_{\neg\ma}$ does not exist s.t. $v +_{M_a} x \in p_{M_a}$}
\State $\sigma_{abs} \gets \sigma_{abs} \cdot \la x\ra $
\If{$x \parallel_{M_a} x\in p_{M_a}$}
\State $\sigma_{abs} \gets \sigma_{abs} \cdot \la x\pr\ra $
\EndIf
\EndIf
\EndFor
\ElsIf{$e \in A_{\neg \ma}$}
\State $\sigma_{abs} \gets \sigma_{abs} \cdot \la e \ra $
\EndIf
\EndFor
\State \rev{$L_{tmp} \gets L_{tmp} + \{\sigma_{abs}\} $ \text{// standard multiset addition}}
\EndFor
\State $A_{\times} \gets \{A\subseteq A_{new} \mid \exists C  \in MDT(G(p_{M_a})), \forall x\in A: \text{x belongs to XOR-complete module } C \}$
\State $L_{tmp} \gets \operatorname{deleteChoiceActivities}(L_{tmp}, A_{\times}) $
\State \Return $L_{tmp}$, $M_a$
\end{algorithmic}
\label{alg:ea1}
\end{algorithm}

To start, we have $L_2 = [\sigma_{ex1}, \sigma_{ex2}, \sigma_{ex3}, \sigma_{ex4}, \ldots]$ (cf. \autoref{tab:log}), $|L_2| = 46$, and $\norm{L_2} = 455$\footnote{Minimal df-complete event logs are computed by \autoref{alg:minimal} in \autoref{sec:min}}. For $\sigma_{ex1}$, $\ea_{mabs}$ computes $A_{\ma} = \{\texttt{CBW}, \texttt{NC}\}$ such that $\sigma_{ex1}[2]$ is the first event abstracted by $\texttt{AB} \in A_{new} = \{\texttt{AB}, \texttt{AC}, \texttt{FDD}\}$ (line 8). However, the condition in line 9 is not true, because $\texttt{RP} +_{M_a} \texttt{AB} \in p_{M_a}$. Also, for $\sigma_{ex1}[3]$ the abstract activity $\texttt{NC}$ is in choice relation to $\texttt{RP}$. Thus, $\sigma_{abs,1} = \la \texttt{RBP}, \texttt{RP}, \texttt{AT}\ra$. 
For $\sigma_{ex2}$, $\ea_{mabs}$ computes $A_{\ma} = \{\texttt{CBW}, \texttt{NC}, \texttt{RFI}, \texttt{BC}, \texttt{PN}, \texttt{CD}, \texttt{PDD} \}$ such that the condition in line 7 becomes true for any $\sigma_{ex2}\pr[2], \ldots, \sigma_{ex2}\pr[8]$. However, only $ \sigma_{ex2}\pr[2], \sigma_{ex2}\pr[3],$ and $\sigma_{ex2}\pr[6]$ for $\texttt{AB},  \texttt{AC},$ and $\texttt{FDD}$ respectively satisfy the condition in line 8. Since no concrete activity $u \in A_{\neg\ma} = \{\texttt{RBP},\texttt{SC}, \texttt{AT} \}$ is in choice relation to an abstract activity $x \in A_{new}$, three new abstract events $\texttt{AB}, \texttt{AC},$ and $\texttt{FDD}$ are added to $\sigma_{abs,2}$ in line 11. Also, abstract activity \texttt{FDD} is in parallel relation to itself, so that $\texttt{FDD}$ is added a second time to $\sigma_{abs2}$. Overall, $\sigma_{abs2} = \la \texttt{RBP},\texttt{AB}, \texttt{AC}, \texttt{FDD},  \texttt{FDD}, \texttt{SC}, \texttt{AT}\ra $. Since in every trace the activity \texttt{CBW} always occurs before \texttt{NC} and both always occur before \texttt{PN} (cf. \autoref{fig:ex}), the next 44 iterations of the for-loop (line 2) always results in the same abstract trace: $\sigma_{abs2} = \ldots = \sigma_{abs46}$.  and $L_{tmp} = [\sigma_{abs1},  \sigma_{abs2}, \ldots, \sigma_{abs46} ]$. Consequently, $L_{tmp} = [\sigma_{abs1}, \sigma_{abs2}^{45}]$ when the for-loop terminates. 

Because no abstract activities $x,y \in A_{new}$ are in choice relation $x +_{M_a} y \in p_{M_a}$ (cf. no XOR-complete module in \autoref{fig:mdt} (c)), it holds that $A_{\times} = \emptyset$ in line 17 such that no abstract activities are deleted in line 18. Hence, $\ea_{bpa}$ returns $L_{tmp}$ without further changes.
In general, $A_{\times}$ contains sets of activities that are in choice relation to each other, i.e., for any $A\in A_{\times}$, all abstract activities $x,y \in A$ are in choice relation. Abstract activities that are in choice relation must not both occur in a trace $\sigma_{abs} \in L_{tmp}$.  
Also, the frequencies of traces in which they occur should be ``similar''\footnote{Same frequency means $\forall \sigma_{abs}\in L_{tmp}, A\in A_{\times} : | A_{\sigma_{abs}}\cap A | \leq 1 $ and $\forall x,y \in A, A \in A_{\times}:$   $|\operatorname{freq}_x(L_{tmp}) - \operatorname{freq}_y(L_{tmp})| \leq |A|$ with $\operatorname{freq}_x(L) =  |\{ \sigma_{abs} \in L\mid  x \in \sigma_{abs}\}| $}. 
While it is important that not always the same abstract activity is deleted from a trace for proving correctness in \autoref{sec:proof}, similar frequencies ensure that distributions like trace frequencies are maintained as faithfully as possible. The function $\operatorname{deleteChoiceActivities}$ ensures that the aforementioned requirements on abstract activities in choice relation are met.
For the result $L_{tmp}$ of $\ea_{bpa}^1$, the order of events in $L_{tmp}$ may not adhere to the order of activities in $M_a$, which is guaranteed through the second step.



\begin{algorithm}
\scriptsize
\caption{\rev{Second step to synchronize $\ma_{bpa}$: $\ea^2_{bpa}$ }}
\begin{algorithmic}[1]
\Require: Preliminary abstracted event log $L_{tmp}$, Abstracted process tree $M_a$
\Ensure: Abstracted, df-complete event log $L_r$
\State $L_a \gets \Lm(M_a) \hspace{1em}\text{// take the minimal df-complete log as reference (cf. \autoref{alg:minimal})}$ 
\State $L_r \gets [\,]$ \text{// initialize empty log for the result}
\vspace{0.4mm}
\State $L_a/\!\sim \;\,\gets \{L\pr \subseteq L_a \mid \forall \sigma_1, \sigma_2\in L\pr: \delta_{kendall}(\sigma_1, \sigma_2) \neq \bot \}$
\vspace{0.4mm}
\State $L_{tmp}/\!\sim \;\,\gets \{L\pr \subseteq L_{tmp} \mid \forall \sigma_1, \sigma_2\in L\pr: \delta_{kendall}(\sigma_1, \sigma_2) \neq \bot \}$
\ForAll{$ L_a^{class} \in L_a/\!\sim$}
\ForAll{$ L_{tmp}^{class} \in L_{tmp}/\!\sim$}
\vspace{0.4mm}
\If{$\exists \sigma\in L_a^{class}, \sigma_{abs} \in L_{tmp}^{class}: \delta_{kendall}(\sigma, \sigma_{abs})\neq \bot$}
\vspace{0.4mm}
\State $n_1, \ldots, n_{|L_a^{class}|} \gets \operatorname{evenSplitSizes}(|L_{tmp}^{class}|, |L_a^{class}|)$
\ForAll{$ j, \sigma \in \operatorname{enumerate}(L_{a}^{class}) $}
\vspace{0.4mm}
\State $L_{tmp}^{n_j} \gets \operatorname{closestTraces}_{\delta_{kendall}}(n_j, L_{tmp}^{class}, \sigma )$
\vspace{0.4mm}
\State $L_{tmp}^{class} \gets L_{tmp}^{class} - L_{tmp}^{n_j} $
\vspace{0.4mm}
\State $L_{tmp}^{transposed} \gets  \operatorname{transposeAll}_{\delta_{kendall}}(L_{tmp}^{n_j}, \sigma )$
\State $L_r \gets L_r + L_{tmp}^{transposed} $
\EndFor
\EndIf
\EndFor
\EndFor
\State \Return $L_r$
\end{algorithmic}
\label{alg:ea2}
\end{algorithm}
\vspace{-0.25em}

\noindent \textbf{Transposing Events To Ensure Correct Orders.} The second step $\ea^2_{bpa}$ establishes the correct order of events in $L_{tmp}$ with respect to the order of events in the reference minimal df-complete event log $L_a$ of \rev{$M_a$} (cf. \autoref{alg:minimal}).
\rev{The correct order is established through the \emph{Kendall Tau Sequence Distance} $\delta_{kendall}$ \cite{kendall_distance_2020} that computes the minimal number of transpositions needed to transform one trace $\sigma_1$ into the other $\sigma_2$. As $\delta_{kendall}$ requires that both traces are permutations of the same multiset of activities, it is undefined otherwise ($\delta_{kendall} = \bot$). Put succinctly, $\ea^2_{bpa}$ finds a matching between equivalence classes of traces in the reference $L_a$ (line 3) and equivalence classes of traces in the input $L_{tmp}$ (line 4) where equivalence is defined modulo transposition (cf. \autoref{sec:correct}). To maintain relative trace multiplicities within an equivalence class of the reference (line 8-9), the traces in the matched equivalence class of the input are evenly split (line 10-11) and transposed (line 12) to exhibit the same order of events as the reference traces (line 13).}

To continue the illustration, $\Lm$ computes $L_a = [\sigma_1, \ldots, \sigma_4]$ with $\sigma_{1} = \la \texttt{RBP}, \texttt{RP}, \texttt{AP}\ra, $ $ \sigma_2 = \la \texttt{RBP},\texttt{AB}, \texttt{AC}, \texttt{FDD},  \texttt{FDD}, \texttt{SC}, \texttt{AP}\ra, $ $ \sigma_3 = \la \texttt{RBP},\texttt{AB}, \texttt{FDD}, \texttt{AC}, \texttt{FDD}, \texttt{SC}, \texttt{AP}\ra$ and $\sigma_4 = \la \texttt{RBP},\texttt{AB}, \texttt{FDD}, \texttt{FDD}, \texttt{AC}, \texttt{SC}, \texttt{AP}\ra$. Hence, $\Ltr = \{[\sigma_{abs1}], [\sigma_{abs2}^{45}]\}$ (line 3) and $\Lar = \{[\sigma_{1}], [\sigma_2, \sigma_3, \sigma_4]\}$ (line 4) are the two quotient sets modulo transposition. In the first iteration of the for-loop in line 5, the equivalence class $\Lac = [\sigma_1]$ and the equivalence class $\Ltc = [\sigma_{abs1}]$ satisfy the condition in line 7. Therefore, the one trace of $\Ltc$ is evenly split to the one trace of $\Lac$, i.e., $n_{|\Lac|} = n_1 = 1$. Next, the single closest trace $\sigma_{abs1}$ of $\Ltc$ is assigned to $L_{tmp}^{n_1}$ (line 10), removed from the equivalence class $\Ltc$ (line 11), transposed according to the zero distance of $\delta_{kendall}$ between $\sigma_{abs1}$ and $\sigma_1$ (i.e., no transposition is applied), and assigned to the result: $L_r = [\sigma_{abs1}]$. 

In the second iteration of the for-loop in line 5, the equivalence class $\Lac = [\sigma_2, \sigma_3, \sigma_4]$ is matched with the equivalence class $\Ltc = [\sigma_{abs2}^{45}]$ by satisfying the condition in line 7. The 45 traces in $\Ltc$ are evenly split across the 3 traces of $\Lac$ in line 8: $n_1 = 15, n_2 = 15, $ and $n_3 = 15$. Because $45$ can be divided without remainder, each split size $n_1, \ldots$ is of equal size and the remainder does not have to be spread across the splits. The enumeration in line 9 of $\Lac$ results in $j \in \{1,2,3\}$. Hence, $\operatorname{closestTraces}$ finds the 15 closest traces $\sigma \pr \in \Ltc$ ($j=1$, $n_1 = 15$, $\sigma = \sigma_2$) that have the smallest distance $\delta_{kendall}(\sigma\pr, \sigma_2)$ to the current trace $\sigma_2$ of $\Lac$ (line 10). Because all traces in $\Ltc$ have the same distance of 0 to $\sigma_2$, 15 traces of $\Ltc$ are assigned to $L^{n_1}_{tmp} = [\sigma_{abs}^{15}]$. Subsequently, $L^{n_1}_{tmp}$ is removed from $\Ltc$ (line 11), no transpositions are applied (line 12), and $L_r = [\sigma_1, \sigma_{2}^{15}]$. 

For the second iteration of the for-loop in line 9 ($j=2$, $n_2 = 15$, $\sigma = \sigma_3$), $\operatorname{closestTraces}$ finds the next 15 closest traces $\sigma \pr \in \Ltc$ that have the smallest distance $\delta_{kendall}(\sigma\pr, \sigma_3) = 1$ to the current trace $\sigma_3$ of $\Lac$ (line 10). Because all 30 traces in $\Ltc$ have the same smallest distance of 1 to $\sigma_3$, another 15 traces are assigned to $L^{n_2}_{tmp}$ (line 10) and removed from $\Ltc$ (line 11). As the distance greater than zero, the respective transposition to transform each $\sigma\pr \in L^{n_2}_{tmp}$ into $\sigma_2$, i.e., transposing the $\texttt{AC}$ with the directly-following $\texttt{FDD}$, are applied in line 12. Hence, $L_{tmp}^{transposed} = [\la \texttt{RBP},\texttt{AB}, \texttt{FDD}, \texttt{AC}, \texttt{FDD}, \texttt{SC}, \texttt{AP}\ra^{15}]$ and $L_r = [\sigma_{abs1}, \sigma_{abs2}^{15}, \la \texttt{RBP},\texttt{AB}, \texttt{FDD}, \texttt{AC}, \texttt{FDD}, \texttt{SC}, \texttt{AP}\ra^{15}]$ in line 13. Analogously, the third iteration applies the two transpositions to transform each $\sigma_{abs2}$ of the 15 remaining traces in $\Ltc$ into trace $\sigma_4$, resulting in $L_r = [\sigma_{abs1}, \sigma_{abs2}^{15}, \la \texttt{RBP},\texttt{AB}, \texttt{FDD}, \texttt{AC}, \texttt{FDD}, \texttt{SC}, \texttt{AP}\ra^{15}, \la \texttt{RBP},\texttt{AB}, \texttt{FDD}, \texttt{FDD}, \texttt{AC}, \texttt{SC}, \texttt{AP}\ra^{15}]$.
Obviously, the resulting abstracted event log $L_r$ is df-complete wrt. $M_a$: $L_a \subseteq L_r$.

\rev{Importantly, only the \emph{transpose} edit operation on traces for swapping the order of two directly-following events in a trace is allowed, i.e., any other distance metric on traces cannot be applied here. Since $\ea_{bpa}^1$ substitutes and deletes concrete events of abstract activities, inserts abstract activities for self-loops, and deletes events that are in choice order, it already applies the \emph{substitute}, \emph{delete}, and \emph{insert} edit operations in a controlled manner. Consequently, the perfect matching of equivalence classes between $L_a/\!\sim$ and $L_{tmp}/\!\sim$ uniquely exists, because the determinants for the number of equivalence classes in both quotient sets are aligned by $\ea_{bpa}^1$.} 
 
\rev{First, choice relations are satisfied in the previous step and no loops other than the self-loop can occur and are included in $L_{tmp}$ (cf. \autoref{def:restricted}). Second, $L_{tmp}$ and $L_a$ share exactly the same activities. Third, $L_a$ has fewer traces and is smaller than any concrete event log $L$ from which $L_{tmp}$ was abstracted (cf. \autoref{lemma:large}), so no equivalence class of $L_a$ can have more traces than the corresponding equivalence class in $L_{tmp}$. Finally, we point out that transpositions render the timestamp attribute incorrect. To heal the timestamp after transposition, we can either find a timestamp that adheres to the new ordering of events among the concrete events in the ``concrete'' event attribute (cf. last step) or we must interpolate it and flag it as artificially-created accordingly to avoid confusion during process intelligence.}
 
\rev{To sum up, we define $\ea_{bpa} = \ea^2_{bpa} \circ \ea^1_{bpa}$. In the next section, we prove that $\ea_{bpa}$ composed of $\ea_{bpa}^1$ and $\ea_{bpa}^2$ synchronizes $\ma_{bpa}$ under restrictions. To that end, we prove the correctness of our design: $\ea_{bpa}$ returns df-complete event logs $L_a \subseteq L_r$, i.e., both have the same DFG.}

\section{Theoretical Foundations}
\label{sec:theory}
This section establishes the theoretical foundations needed to prove synchronization correctness. 
We define process tree classes that support isomorphic rediscoverability (\autoref{sec:classes}), propose an algorithm $\gen$ to generate minimal df-complete event logs with size metrics (\autoref{sec:min}), establish conditions under which our approach produces well-defined results in (\autoref{sec:conditions}), and prove that these conditions actually establish well-defined results (\autoref{sec:well}).
For full proofs, we refer to the respective lemma in the appendix \autoref{app:lemma}. 

\subsection{Process Tree Classes and Rediscoverability}
\label{sec:classes}


\srm{The first step for proving synchronizability is to connect the levels of model abstraction and event logs based on isomorphic rediscoverability. As stated in \autoref{sec:back}, process tree M is isomorphic rediscoverable by a process discovery algorithm pd (in this work IM) from event log $L$ iff $M'=pd(L)$ is isomorphic to $M$. In the following, we will show that any abstracted process tree $M_a=\ma_{bpa}(M)$ is isomorphic rediscoverable under certain model restrictions. These restrictions are summarized by classes of process trees, i.e., $C_c$ and $C_a$ in Def. \ref{def:cbpa} such that $M_a$ always satisfies the model restrictions of class $C_a$.}

\begin{definition}[Class $C_{c}, C_{a}$]
\label{def:cbpa}
    $\oplus \,(M_1, \ldots, M_n)$ denotes a node at any position in process tree $M$. $M$ is in class $C_{c}$ iff restrictions 1. and 2. are met and in class $C_{a}$ iff restrictions 1.--3. are met:
    \begin{enumerate}
        \item $M$ has no duplicate activities, i.e., $\forall\; i \neq j: A_{M_i} \cap A_{M_j} = \emptyset $,
        \item If $\oplus = \;\circlearrowleft$, then the node is a self-loop, i.e., $\tloop(v, \tau)$ for some activity $v \in \AM$ (i.e., any other loop $\circlearrowleft(M_1, \ldots M_n)$ is prohibited),
        \item No $\tau$'s outside of the self-loops are allowed: If $\oplus \neq \;\circlearrowleft$, then $\forall\; i \leq n: M_i \neq \tau$.
    \end{enumerate}
\end{definition}

$M_a$ in the running example is in class $C_a$. 
These classes differ from standard IM restrictions regarding loop and $\tau$ handling, requiring separate rediscoverability proofs.
Note that we number the lemmata in this extended version according to our main paper, i.e., lemma 1 and 2 of the main paper are similarly numbered in the extended version.

\setcounter{lemma}{2} 

\begin{lemma}[Process trees in $C_{a}$ are isomorphic rediscoverable]
\label{lemma:redisc}
\srm{Let $M$ be a process tree and $L$ be an event log. }
If $M$ is in class $C_{a}$ and $M$ and $L$ are df-complete (cf. \autoref{sec:back}), i.e., $M \dfc L$, then $\disc_{IM}$ discovers a process tree $M\pr$ from $L$ that is isomorphic to $M$. 
\end{lemma}

The proof strategy is to distinguish whether $M$ in $C_a$ contains a self-loop or not. If $M$ does not contain a self-loop, $M$ adheres to the restrictions in \cite{leemans_tech_2013}. If $M$ contains a self-loop, we extend the base case of the induction in Theorem 14 \cite{leemans_tech_2013} to also hold for any splitted log $L_v \subseteq [\la v,v \ra^m, \la v,v,v\ra^n, \ldots]$ for which the ``Strict Tau Loop'' fall through discovers the self-loop. Since any $M$ in $C_a$ is isomorphic rediscoverable, what is left to prove is that $M_a = \ma_{bpa}(M)$ is in class $C_a$.

\begin{lemma}[$\ma_{bpa}$ abstracted process trees are in $C_a$]
\label{lemma:redisc2}
    If $\ma_{bpa}$ is applicable to process tree M, then $M_a = \ma_{bpa}(M) $ is in class $C_a$.
\end{lemma}

The main idea of the proof lies in the inability of a behavioral profile $p_{M_a}$ to distinguish whether activities are in a $\wedge$-node or in a $\circlearrowleft$-node. $\ma_{bpa}$ handles the inability by always synthesizing a $\wedge$-node for AND-complete modules in the $MDT(G)$. The only $\tau$ in $M_a$ can occur due to the additional step that adds a self-loop node to $M_a$ (cf. step 3 in \autoref{sec:bpa}). Overall, \autoref{lemma:redisc} and \autoref{lemma:redisc2} together imply isomorphic rediscoverability of $M_a$. 
Because the isomorphic rediscoverability of $M_a$ is conditioned on df-complete event logs $L_a$, we generate $L_a$ given $M_a$ using the semantics $\mathcal{L}(M_a)$ in the next section.

\subsection{Minimal, Directly-follows Complete Event Logs}
\label{sec:min}

Given a process tree $M$ either in $C_c$ or in $C_a$, there exist countably infinite many df-complete event logs $L$ due to the self-loop node. However, an EA $\ea_{bpa}$ must reduce the size of the event log. Therefore, we generate the \emph{minimal}, df-complete (mdf-complete) event log $\Lm(M_a)$ of the countably infinite set of df-complete event logs as a target for our EA technique. Because there are no further $\circlearrowleft$-nodes in $M$ other than self-loop nodes, minimality of $|\Lm(M_a)|$ is equivalent to minimality of $\norm{\Lm(M_a)}$.
Hence, $\Lm(M_a)$ can be easily computed using the recursive definition of $M$'s language $\mathcal{L}(M)$ for $\oplus$-nodes with $\oplus \in \{\times,\rightarrow,\wedge\}$ and leaves $M = \tau$ or $M = v$. The only difference of $\Lm(M_a)$ compared to $\mathcal{L}(M)$ is the case for the self-loop node $\tloop(v,\tau)$ that simply assigns the trace $\la v,v\ra$.

\begin{algorithm}[H]
\begin{scriptsize}
\caption{Computing trace number and lengths of $\Lm(M)$: $\gen$}
\label{alg:minimal}
\begin{algorithmic}[1]
\Require Process tree M in $C_c$
\Ensure Number of traces $|\Lm(M)|$ and sequence of trace lengths $\operatorname{lens}(M) = \operatorname{lens}(\Lm(M))$
\If {$M$ = $\tau$ }
\State \Return $|\Lm(M)|$ $\gets$ 1, $\operatorname{lens}(M) \gets$ $\langle 0 \rangle$
\ElsIf {$M$ = v }
\State \Return  $|\Lm(M)|$ $\gets$ 1, $\operatorname{lens}(M) \gets$ $\langle 1 \rangle$
\ElsIf {$M$ = $\tloop(v, \tau)$ for some $v \in \AM$}
\State \Return $|\Lm(M)|$ $\gets$ 1, $\operatorname{lens}(M) \gets$ $\langle2\rangle$
\ElsIf {$M$ = $\times (M_1, \ldots, M_n)$}
\State \Return 
\vspace{-1.7em}
\begin{align*} |\Lm(M)| \gets \sum_{i=1}^{n} |\Lm(M_i)|,\; \operatorname{lens}(M) \gets  \bigodot_{i=1}^n\operatorname{lens}(M_i), \\
\shortintertext{\hspace{3.4em}//where $\bigodot$ concatenates an ordered collection of sequences}
\end{align*}
\vspace{-3.3em}
\ElsIf {$M$ = $\rightarrow (M_1, \ldots, M_n)$}
\State \Return \vspace{-1.7em}
\begin{align*}
    |\Lm(M)| \gets \prod_{i=1}^{n} |\Lm(M_i)|,  \;
    \operatorname{lens}(M) \gets \bigodot_{k=1}^{|\Lm(M)|} <\sum_{i = 1}^n \operatorname{lens}(M_i)[\iota_{k,i}]>
    \shortintertext{\hspace{3.4em}//where $\iota$ is a bijection $\iota:\{1, \ldots, |\Lm(M)|\} \rightarrow\bigtimes_{i=1}^{n}\{1,\ldots, |\operatorname{lens}(M_i)|\}$ and}
                   \shortintertext{\hspace{3.4em}//$\iota_{k,i}= \pi_i(\iota(k))$ selects the $i$th element of $\iota(k) = (l_1, \ldots, l_n)$. }
\end{align*}
\vspace{-3.3em}
\ElsIf {$M$ = $\wedge (M_1, \ldots, M_n)$}
\State \hspace{-0.3em}\Return \vspace{-1.3em}
\begin{align*}
    |\Lm(M)| &\gets \sum_{k=1}^{\prod_{i=1}^{n} |\Lm(M_i)|}\genfrac{(}{)}{0pt}{0}{m_k}{\:\bigodot_{i=1}^{n}  \la \operatorname{lens}(M_i)[\iota_{k,i}] \ra \:},  \\
    \shortintertext{\hspace{3.4em}//where $\binom{m}{\,\la l_1,\ldots,l_n\ra\,} = \frac{m!}{l_1! * \ldots l_n!}$ is the multinomial coefficient with $m = \sum_{i=1}^n l_i$ and}
    \shortintertext{\hspace{3.4em}//$\iota$ is a bijection $\iota:\{1, \ldots, \prod_{i=1}^{n}|\Lm(M_i)|\} \rightarrow\bigtimes_{i=1}^{n}\{1,\ldots, |\operatorname{lens}(M_i)|\}$}
    \operatorname{lens}(M) &\gets \bigodot_{k=1}^{\prod_{i=1}^{n}|\Lm(M_i)|} <\sum_{i = 1}^n \operatorname{lens}(M_i)[\iota_{k,i}]>^{\mnomc{\bigodot_{i=1}^n \la \operatorname{lens}(M_i)[\iota_{k,i}]\ra}{m_k}}\\
\end{align*} 
\vspace{-3em}
\EndIf
\end{algorithmic}
\end{scriptsize}
\end{algorithm}

The semantics $\mathcal{L}(M)$ neither provide the number of traces $|\mathcal{L}(M)|$ nor the size $\norm{\mathcal{L}(M)}$ generated for arbitrary process trees $M$ in $C_c$. An existing algorithm $\Lm\pr$ for computing the number of traces $|\mathcal{L}(M)|$ in \cite{janssenswillen_calculating_2016} is limited compared to our proposed $\gen$ (\emph{n}umber of \emph{t}races and their \emph{l}engths) in \autoref{alg:minimal} for two reasons.
First, $\gen\pr$ does not generate the sequence of lengths $\la |\sigma_1|, \ldots, |\sigma_k|\ra = \operatorname{lens}(\Lm(M))$ of the $k = |\Lm(M)|$ traces $\sigma_1, \ldots, \sigma_k \in \Lm(M) $ required to compute the log size. Second, $\Lm\pr$ is limited to $\wedge$-nodes with fixed lengths of traces in the language of its children such that $\times$-nodes that result in varying lengths cannot be arbitrarily nested with $\wedge$-nodes. Although \cite{janssenswillen_calculating_2016} fix this limitation by transforming the process tree, the resulting process tree duplicates labels such that it violates the restrictions of $C_c$ and $ C_{a}$, increases the size, and is not isomorphic rediscoverable. 


We illustrate $\Lm(M)$ and $\gen$ with $M_a$ of the running example (cf. \autoref{fig:ex}). Both recur until either a leaf (line 2 and 4) or a self-loop $\tloop(x, \tau) $ (line 6) is reached. Hence, for each of the six leaves $M\pr$ with activities \texttt{RBP}, \texttt{RP}, $\ldots$, $\Lm(M)$ and $\gen$ reach line 4, e.g., for $M_1 = \texttt{RBP}$, we have $\Lm(M_1)$ = $[\langle\texttt{RBP}\rangle]$, $|\Lm(M_1)|$ = 1, and $\operatorname{lens}(M_1) = \la1\ra$. For the self-loop $M_7 = \:\circlearrowleft\hspace{-0.3em}(\texttt{FDD}, \tau)$, we have $\Lm(M_7) = [\langle\texttt{FDD}, \texttt{FDD}\rangle]$, $|\Lm(M_7)| = 1$, and $\operatorname{lens}(M_7) = \langle 2\rangle$. Next, for $M_8 = \wedge(\texttt{AC}, M_7)$, we have $\Lm(M_8) = [\la \texttt{AC}, \texttt{FDD}, \texttt{FDD}\ra, \la \texttt{FDD}, \texttt{AC}, \texttt{FDD} \ra, \la \texttt{FDD}, \texttt{FDD}, \texttt{AC}\ra]$, $|\Lm(M_8)|$ =\\ $\sum_{k=1}^1 \genfrac{(}{)}{0pt}{1}{3}{\: \la \operatorname{lens}(M_1)[\pi_1(1,1)], \,\operatorname{lens}(M_2)[\pi_2(1,1)] \ra \:} = \frac{3!}{1!2!}=3$, and $\operatorname{lens}(M_8) = \la 3, 3, 3\ra$. For $M_9 = \,\rightarrow(\texttt{AB}, M_8, \texttt{SC})$, we have $\Lm(M_9) = [\la \texttt{AB}, \texttt{AC}, \texttt{FDD}, \texttt{FDD}, \texttt{SC}\ra, \ldots]$, $|\Lm(M_9)|$ = 3, and $\operatorname{lens}(M_9) = \la5,5,5\ra$. Next, for $M_{10}=\times(\texttt{RP}, M_9)$, we have $\Lm(M_{10}) = [\la \texttt{RP}\ra, \ldots ]$, $|\Lm(M_{10})|$ = 4, and $\operatorname{lens}(M_{10}) = \la 1,5,5,5\ra$. Finally, $\Lm$ and $\gen$ return $\Lm(M_a) = [\la \texttt{RBP}, \texttt{RP}, \texttt{AP}\ra, \ldots]$, $|\Lm(M_a)|$ = 4, and $\operatorname{lens}(M_a) = \la 3, 7,7,7\ra$ (cf. \autoref{sec:sea} for the full event log).

Since we know the number of traces $|\Lm(M_a)| $ and the lengths $\operatorname{lens}(M_a)$, we can compute the size $\norm{\Lm(M_a)}$ by summing the trace lengths to yield $\norm{\Lm(M_a)} = 24$. In \autoref{alemma:gensize}, we prove the correctness of $\gen$.





\begin{lemma}[Number of traces and size of $\Lm(M).\text{log}$]
\label{lemma:gensize}
    If $M$ is in $C_c$, then $\Lm(M).\text{tr} = |\Lm(M).\text{lens}|$ and $ \Lm(M).\text{\emph{log}} = \Lm(M)$ with $|\Lm(M)| = \gen(M).\text{\emph{tr}}$ traces and size $\norm{\Lm(M)} = \sum_{k=1}^{\gen(M).\text{tr}} \gen(M).\text{\emph{lens}}[k] $. 
\end{lemma}
The proof strategy is to sketch the reasoning for the induction step of a structural induction on process tree $M$ in $C_c$. The reasoning for the trace lengths of a $\rightarrow$-node is that the lengths of concatenated traces from children is the sum of the respective children's trace lengths as indexed by $\iota$. The reasoning for the number of traces of a $\wedge$-node is to characterize interleaving $\prod_{i=1}^{n} \gen(M_i).\text{tr}$ different trace combinations of varying trace lengths that are indexed by $\iota$ as shuffling of $n$ card decks \cite{aldous_shuffling_1986}. Shuffling the first two traces and then iteratively shuffling the next trace into each existing interleaving results in as many interleavings as the multinomial coefficient computes. Similar to $\rightarrow$-nodes, trace lengths are computed except that each individual trace length is repeated as often as there are interleavings corresponding to the $k$th combination of trace lengths. 

It follows that we can compute the number of traces and the size of $\Lm(M_a)$ for any abstracted process tree $M_a$. Given how the number of traces and the size of $\Lm(M)$ are computed, we order $\oplus$-nodes with respect to both their number of traces and sizes. 

\begin{lemma}[Operator ordering wrt. their mdf-complete log]
\label{alemma:sort}
    Let $M = \\\oplus(M_1, \ldots, M_n)$ with $\oplus \in \{\wedge, \rightarrow, \times\}$ be three process trees in $C_c$.
If all children $M_i$ of $M$ have at least two traces, i.e., $|\Lm(M_i) | \geq 2, i \in \{1, \ldots, n\}$ then: $|L_{\times}| < |L_{\rightarrow}| < |L_{\wedge}|$ and $ \norm{L_{\times}} <\norm{L_{\rightarrow}} <  \norm{L_{\wedge}}$ with $L_{\oplus} = \Lm(\oplus(M_1, \ldots, M_n))$. 
If all children $M_i$ of $M$ have between one and two traces, i.e., $|\Lm(M_i) | \in \{ 1, 2\} , i \in \{1, \ldots, n\}$, then
$ |L_{\times}| \leq |L_{\wedge}|$, $|L_{\rightarrow}| <  |L_{\wedge}|$, and $ \norm{L_{\times}} \leq \norm{L_{\rightarrow}}  <  \norm{L_{\wedge}}$.
\end{lemma}
\begin{proof}
    (Sketch) Let $k_1, \ldots, k_n $ be the number of traces in the $M_1, \ldots, M_n$ children, i.e., $k_i = |\Lm(M_i)|$. For $k_i \geq 2$, it holds that $\sum_{i=1}^n k_i < \prod^n_{i=1}k_i < \\\sum_{k=1}^{\prod_{i=1}^{n} k_i}\genfrac{(}{)}{0pt}{1}{m_k}{\:\bigodot_{i=1}^{n}  \la \text{lens}(M_i)[\iota_{k,i}] \ra \:}$, because multiplication grows ``faster'' than summation and the sum of $ \prod_{i=1}^{n} k_i$ multinomial coefficients ``faster'' than $\prod_{i=1}^{n} k_i$ itself. Hence, $|L_{\times}| < |L_{\rightarrow}| < |L_{\wedge}|$. Because the children $M_i$ are the same for $L_{\times}, L_{\rightarrow},$ and $L_{\wedge}$, their trace lengths $\text{lens}(M_i)$ are the same such that $ \norm{L_{\times}} <\norm{L_{\rightarrow}} <  \norm{L_{\wedge}}$ follows. 

    \noindent For $k_i \in \{1,2\}$, the sum can grow ``faster'' than the product: $\sum_{i=1}^n k_i > \prod^n_{i=1}k_i $, e.g., if all $k_i$ are equal to 1. Still, $\norm{L_{\times}} \leq \norm{L_{\rightarrow}}$, because even if the number of traces is larger for $L_{\times}$, all the events in $e \in L_{\times}$ at least occur once in $L_{\rightarrow}$ in fewer traces instead of in their individual traces. The sum of $ \prod_{i=1}^{n} k_i$ multinomial coefficients can be equal to the $\sum_{i=1}^n k_i$, e.g., if all $k_i$ are equal to 1. Nevertheless, the sum $\sum_{i=1}^n k_i$ does not exceed the sum of multinomial coefficients: $\sum_{i=1}^n k_i \leq \sum_{k=1}^{\prod_{i=1}^{n} k_i}\genfrac{(}{)}{0pt}{1}{m_k}{\:\bigodot_{i=1}^{n}  \la \text{lens}(M_i)[\iota_{k,i}] \ra}$. $\prod^n_{i=1}k_i $ only ``grows'', if $k_i = 2$ occurs ``often''. Yet, $k_i = 2$ occurring ``often'' results in a large factorial in the multinomial coefficient. Hence, $ \prod^n_{i=1}k_i < \sum_{k=1}^{\prod_{i=1}^{n} k_i}\genfrac{(}{)}{0pt}{1}{m_k}{\:\bigodot_{i=1}^{n}  \la \text{lens}(M_i)[\iota_{k,i}] \ra}$. To sum up, $ |L_{\times}| \leq |L_{\wedge}|$ and $|L_{\rightarrow}| <  |L_{\wedge}|$. From $|L_{\rightarrow}| <  |L_{\wedge}|$ and the observation that the trace lengths in $L_{\rightarrow}$ and $L_{\wedge}$ are equal except for their multiplicities in $\gen(M).\text{lens}$ (cf. Algorithm 3 line 10 and line 12 in the main paper), it follows that $\norm{L_{\rightarrow}} < \norm{L_{\wedge}}$. Even if $ |L_{\times}| = |L_{\wedge}|$, still $\norm{L_{\times}} < \norm{L_{\wedge}}$, because the $\text{lens}(\wedge(M_1, \ldots,M_n)$ in line 12 of Algorithm 3 in our main paper (cf. $\operatorname{ntl}$ Algorithm 3) adds the children's trace lengths $\sum_{i = 1}^n \text{lens}(M_i)[\iota_{k,i}] $ before repeating them as often as there are interleavings. 
\end{proof}

The $\oplus$-node ordering is important for synchronizability. Any non-order-pres\-erving MA must ensure that abstract activities in $M_a$ are not children of a greater $\oplus$-node than the $\oplus\pr$-node of $M$ in which the corresponding abstracted activities are children of to avoid that the mdf-complete event log $\Lm(M_a)$ becomes larger than the mdf-complete event log $\Lm(M)$. Hence, if a MA technique fails to consider the $\oplus$-node ordering, there does not exist a EA technique that synchronizes.
In the next section, we present the restrictions necessary to ensure that the mdf-complete event log $\Lm(M_a)$ of $M_a = \ma_{bpa}(M)$ is smaller than the mdf-complete event log $\Lm(M)$ of $M$. 

\subsection{Event Log Restrictions}
\label{sec:conditions}

For our EA technique to be well-defined, the df-complete event log $L_a$ for the abstracted process model must not be larger than what the concrete event log $L$ can support.

\begin{definition}[Restricted event log] 
\label{def:restricted}
Event log $L$ is restricted iff $\disc_{IM}$ discovers a process tree $M = \disc_{IM}(L)$ such that: 
\begin{enumerate} 
    \item \textbf{Fall throughs:} $\disc_{IM}$ has only executed the ``Strict Tau Loop'' to discover a self-loop $\tloop(v, \tau)$ for $v \in \AL$ and the ``Empty Traces''\footnote{Hence, no nested tau loops $\tloop(M_1, \tau), M_1 \neq v$ are discovered by ``Strict Tau Loop'' and no fall throughs ``Activity Once Per Trace'', ``Activity Concurrent'', ``Tau Loop'', and ``Flower Model'' are executed.},
    \item \textbf{Cuts:} $\disc_{IM}$ has only executed choice, parallel, and sequence cuts\footnote{Hence, no loop cut is found.},
    \item \textbf{Base cases:} $\disc_{IM}$ has executed any number of base cases, and
    \item \textbf{Model structure:} If $\oplus (M_1, \ldots, M_n)$ is a node in $M$ with at least one child being an activity or a self-loop, i.e., $M_i = v$ or $M_i = \tloop(v, \tau)$ for $v \in \AL $ and $i\in \{1, \ldots, n\}$, then $\oplus = \times$ or $\oplus = \wedge$.
\end{enumerate}
\end{definition}

We motivate and illustrate each restriction imposed on a event log $L$ by giving relaxed restrictions and counterexamples. 


\noindent \textbf{Fall throughs} \\
\textsl{``Activity Once Per Trace''}: The $\operatorname{syn}$ (cf. step 3 Section 3.2 in our main paper and Algorithm 3.2 in \cite{smirnov_businessm_2012}) of process trees from an abstracted behavioral profile $p_{M_a}$ can be the reason for larger mdf-complete event logs $L_a = \Lm(M_a)$. 
The ``Activity Once Per Trace'' can discover process trees $M$ whose mdf-complete event logs significantly exceed the number of traces and size of the original event log $L$. Therefore, the IM must not execute this fall through during discovery of the abstracted process tree $M_a$ from an abstracted event log $L_a$.

\noindent \textsl{Counterexample}: $L = \{\la a,b\ra, \la e,c,d\ra, \la d,e,c\ra\}$,  
$\disc_{IM}(L) = M = \times(\rightarrow(a, b), \wedge(e, \wedge(c,d)))$ Obtain
$\ma_{bpa}(M) = M_a = \:\times(x, \wedge(c,d,e) ) $ for $\agg(x) = \{a,b\}$ and $ w_t \leq 0.5$ with $L_a = \{\la x\ra, \la c,d,e\ra, \la c,e,d\ra, \ldots, \la e,d,c\ra\} $. The number of traces $|L_a| = 7$ and the size is $\norm{L_a}= \sum_{k=1}^{7} \la 1,3,3,3,3,3,3\ra[k] = 19$ and, thus, both the number of traces and the size of $L_a$ exceed that of $L$ respectively. Hence, no $\ea$ exists that can synchronize to $\ma_{bpa}$.

\noindent \textsl{``Activity Concurrent''}: The following counterexamples exploits the property of $\ma_{bpa}$ to change the behavior of the process tree $M$ even for activities that are not abstracted: 

\noindent \textsl{Counterexample}:   $L = \{\la a,b\ra, \la f,c,d,e\ra, \la c,d,e,f,f\ra,  \la c, d,e,f\ra\}$ and 
$\disc_{IM}(L)  = M = \times(\rightarrow(a,b), \wedge(\times(\tloop(f,\tau),\tau), \rightarrow(c,d,e)))$. Set $\agg(x) = \{a,b\}$ and $ w_t \leq 0.5$ to obtain $M_a = \times(x, \wedge(\tloop(b,\tau),\rightarrow(c,d,g)) $ with $ L_a = \{\la x\ra, \la b,b,c,d,g\ra, \la b,c,b,d,g\ra, \ldots, \la g,d,c,b,b\ra\} $. Both the number of traces and the size of $L_a$ exceeds that of $L$ respectively, as $|L_a| = 11 > 4 = |L|$ and $ \norm{L_a} = \sum_{k=1}^{11} \la 1,5,\ldots,5\ra[k] = 51 > 15 = \norm{L}$.

\noindent \textbf{Fall throughs, Cuts (Loop)} \\
\noindent \textsl{``Strict Tau Loop'', ``Tau Loop'', ``Flower Model'', ``Loop Cut''}: In \autoref{lemma:redisc2}, we have established that any $M_a$ is in class $C_{a}$, i.e., does not contain any $\circlearrowleft$-node other than a self-loop due to the indistinguishability of a $\circlearrowleft$-node and a $\wedge$-node in the BP $p_{M_a}$. The same property of BPs implies that any $\circlearrowleft$-node $M_{\circlearrowleft}$ in $M$ is replaced by a $\wedge$-node, if all of the activities $v \in \A_{M_{\circlearrowleft}}$ are not abstracted by $\ma_{bpa}$ ($A_{M_{\circlearrowleft}} \subseteq A_c$ in Definition 2 in the main paper). Through two counterexamples, we show that if $M$ contains a node $M_{\circlearrowleft}$ that is not a self-loop node $\tloop(v, \tau)$ for $v \in \AL$, it is neither guaranteed that $\ma_{bpa}$ is a MA nor is the mdf-complete event log $L_a$ necessarily smaller than $L$ such that synchronizability becomes impossible. To discover a self-loop both for $M$ and for $M_a$, either the ``Strict Tau Loop'' or the ``Tau Loop'' with a restriction on only discovering a self-loop are suitable. Since ``Strict Tau Loop'' exactly matches the $\la v,v\ra$ directly-follows pattern in an event log, we opt for the ``Strict Tau Loop'' in \autoref{def:restricted}.


\noindent \textsl{Tau Loop Counterexample}: $L = \{\la a,c,a\ra, \la b,a\ra, \la d,e\ra\}$ and $M = \disc_{IM}(L) \\=  \times( \circlearrowleft(\times(b,\rightarrow(a, \times(c,\tau))), \tau), \rightarrow(d,e))$.
Apply $\ma_{bpa}$ for $\agg(x) = \{d,e\}$ and $w_t \leq 0.5$ to obtain the abstracted model $M_a =\times( \wedge(\tloop(a,\tau),\tloop(b,\tau), \tloop(c,\tau)), x)$ with $L_a = \{\la a,a,c,c,b,b \ra, \ldots, \la  c,c,b,b,a,a \ra, \la x\ra\}$. Hence, $|L_a| = \mnomc{\la 2,2,2\ra}{6} + 1 = 91 > 3 = |L|$ and $\norm{L_a} = \sum_{k=1}^{91} \la 1,6,\ldots,6\ra[k] = 541 > 7 = \norm{L}$.

\noindent \textsl{Flower Model Counterexample}: $L = \{ \la a,b\ra, \la a,b,c\ra, \la c,a\ra, \la d,e\ra\}$ and $\disc_{IM}(L) = M =  \times( \tloop(\times(a,b,c), \tau), \rightarrow(d,e))$\footnote{``Strict Tau Loop'' is already restricted to only discover a self-loop and ``Tau Loop'' excluded entirely}.
Apply $\ma_{bpa}$ for $\agg(x) = \{d,e\}$ and $w_t \leq 0.5$ to obtain the abstracted model $M_a =\times( \wedge(\tloop(a,\tau),\circlearrowleft(b,\tau), \circlearrowleft(c,\tau)), x)$ with $|M_a|=12 > 10 = |M|$, i.e., $\ma_{bpa}$ is not a MA for process trees $M$ that contain a flower model. 

\noindent \textsl{Loop Cut Counterexample}: $L = \{\la a,b,a \ra, \la c,d,e\ra\}$ and $M =  \disc_{IM}(L) = \\\times( \tloop(a,b), \rightarrow(d, e, f))$. Apply $\ma_{bpa}$ for $\agg(x) = \{c,d\}$ and $w_t \leq 0.5$ to obtain the abstracted model $M_a =\times( \wedge(a,b), \rightarrow(x, f))$ with $L_a = \{\la a,b\ra, \la b,a\ra, \la x,f\ra\}$. Hence, $|L_a| = 3 > 2 = |L| $ and $\norm{L_a} = 6 = \norm{L_a}$.

\noindent \textbf{Model structure} \\
\noindent The $\operatorname{dv}_{w_t, \agg}$ (cf. Algorithm 3.1 in \cite{smirnov_businessm_2012}) has a static order of returning ordering relations for abstracted activities $x$ and $y$: First, a choice order $x +_{M_a} y$ is returned (line 12), then the inverse strict order $x \rightsquigarrow_{M_a}^{-1} y$ (line 15 and line 19) followed by the strict order $x \rightsquigarrow_{M_a} y$ (line 17) and, finally, the parallel order $x \parallel_{M_a} y$ is returned (line 21), as soon as the threshold $w_t$ is below the respective relative frequency and an ``earlier'' order relation is not returned already (\namedlabel{order}{\circled{o}}). As long as a node $M\pr = \:\rightarrow (M_1,\ldots, M_n)$ in $M$ has only children $M_i$ whose number of traces in the corresponding mdf-complete event log $\Lm(M_i)$ are all at least $2$, the static order in $\operatorname{dv}_{w_t, \agg}$ aligns with the order of operators in a process tree as shown in \autoref{alemma:sort}. However, if $M\pr$ has a child $M_i$ whose number of traces is below $2$, returning a choice order $x +_{M_a}y$ can increase the number of traces in $L_a$, as the order of operators with respect to a sequence node and a choice node depends on the children $M_1, \ldots, M_n$ (cf. \autoref{alemma:sort}). As $\operatorname{dv}_{w_t, \agg}$ does not dynamically depend on the children $M_1, \ldots, M_n$ of a node $M\pr$ in the process tree $M$, we must exclude $M\pr$ that has a child $M_i$ with less than two traces.

\noindent \textsl{Counterexample}: $L = \{\la a,b,c \ra\}$ and  $ M = \disc_{IM}(L)= \: \rightarrow(a,b,c)$. Apply $\ma_{bpa}$ for $\agg(x) = \{a,c\}$ and $w_t \leq 0.5$ to obtain $M_a = \times (v,b)$ with $L_a = \{\la v\ra, \la b\ra\}$. Although $\norm{L_a} = 2 < 3 = \norm{L}$, the number of traces increases $|L_a| = 2 > 1 = |L|$. 

 \janik{For example, the model structure restriction prohibits that $L_a$ has more traces than $\Lm(M)$: $|L_a| > |\Lm(M)|$. Because the process tree $M = \disc_{IM}(L)$ discovered from restricted event log $L$ is in class $C_c$ (cf. \autoref{def:cbpa}), we can generate mdf-complete event logs for $M$\footnote{We prove that $M$ is in $C_c$ in \autoref{alemma:semid}.}.} Since $\ma_{bpa}$ can change the order of sequentially-ordered activities ($\rightsquigarrow_M$) to choice-ordered activities ($+_{M_a}$) given that $\operatorname{dv}_{\agg,w_t}(x,y)$ prioritizes $+$ over $\rightsquigarrow$ (cf. \autoref{sec:bpa}), the number of traces in $L_a$ can become larger than the number of traces in $L$: $|L_a| > |L|$ (compare line 8 and line 10 in \autoref{alg:minimal}). The model structure restriction could be avoided, if we allow an EA $\ea$ to split traces of the concrete event log $L$. Nevertheless, an abstraction should intuitively maintain the semantics of the event log while reducing its complexity (in terms of discovered model size). Splitting traces, however, means that during abstraction new process instances can be created, which is clearly not desired during abstraction. \janik{Thus, we can either prohibit the model structure or restrict what activities can be aggregated, i.e., restrict $\agg$, to avoid more traces in $L_a$.}

For example, $M$ in \autoref{fig:ex} violates the model structure restriction, but $\agg$ as depicted constitutes a case for which $|L_a| < |L|$. Yet, a different function $\agg\pr$ that only abstracts the start \texttt{RBP} and end activity \texttt{AT} would result in $|L_a| > |L|$. We opt to prohibit the model structure, because we aim to synchronize $\ma_{bpa}$ with an EA technique and not apply further modifications to $\ma_{bpa}$.
Through the restrictions on an event log, we prove that $L$ can always support the mdf-complete event log $L_a$ required for well-definedness of $\ea_{bpa}$ in the next section.

\subsection{Well-definedness}
\label{sec:well}

In this section, we prove the necessary condition for synchronization: Event log $L$ can never become smaller than $L_a$. If the opposite would hold, we could not define an EA technique, because the EA technique would need to either insert new traces and or events or both into $L$ for df-completeness. 

To prove the necessary condition, we must consider under what parametrization $\agg$ and $w_t$ the $L_a = \Lm(M_a)$ with $M_a = \ma_{bpa}(M)$ becomes maximal. Hence, we prove for what parameters $\Lm(M_a)$ becomes maximal. 

\setcounter{lemma}{6} 

\begin{lemma}[Maximal $\Lm(M_a)$]
\label{lemma:maxoflog}
    If $L$ is restricted and $\ma_{bpa}$ applicable to $M = \disc_{IM}(L)$, $w_t = \wtop$, $\agg(x) = \{v,u\} $, and $ \agg(z) = \{z\} $ with $ z\in A\setminus\{v,u\}$, i.e., $\ma_{bpa}$ aggregates exactly two activities into $x$, then $M_{a} = \ma_{bpa}(M)$ has the most traces and maximal size of all $\Lm(M_a\pr)$ generated for other $M_a\pr$ with $A_c\pr \cap \{v,u\} = \emptyset $ and $|\bigcup_{y\in A_{new}\pr} \agg\pr(y)| > |\bigcup_{y\in A_{new}} \agg(y)|$ (cf. \autoref{def:bpa}).
\end{lemma}

The main idea of the proof is to compare the behavioral profile $p_{M_a}$ with all behavioral profiles $p_{M_a\pr}$, as the respective process trees $M_a$ and $M_a\pr$ are similarly synthesized given the BP. Because $p_{M_a\pr}$ can only contain less concrete activities $q \in A_c$ than $p_{M_a}$ and both $v$ and $u$ must also be abstracted in $p_{M_a\pr}$, the only cause for more traces or events in $\Lm(M_a\pr)$ can be due to different order relations in $p_{M_a\pr}$. Yet, the restricted event log $L$ aligns the priorization of $\operatorname{dv}_{\agg,w_t}$ to return order relations from $+$ to $\parallel$ (cf. \autoref{sec:bpa}) with the operator ordering in \autoref{alemma:sort}. Hence, $p_{M_a\pr}$ can only contain order relations that imply less traces and events in $\Lm(M_a\pr)$. 

\setcounter{lemma}{0} 

Given the restrictions on event log $L$, we show that $L_a$ for the parameters that maximize it, is always smaller than $L$. 

\begin{lemma}[$\Lm(M_a)$ is smaller]
\label{lemma:large}
    If $L$ is restricted and $\ma_{bpa}$ is applicable to $M = \disc_{IM}(L)$, then $L_{a} = \Lm(M_a)$ 
    has fewer traces and is smaller than $L$: $|L_{a}|< |L|$ and $\norm{L_{a} }< \norm{L}$.
\end{lemma}

The proof strategy is to apply induction on the size of $M = \disc(L)$ and compare the mdf-complete event log $\Lm(M)$ with the maximal mdf-complete event log $\Lm(M_a)$ for $M_a$ abstracted through parameters $\agg(x) = \{v,u\}$ and $w_t = \wtop$. \janik{Only abstracting two concrete activities and setting $w_t = \wtop$ maximizes $\Lm(M_a)$ (cf. \autoref{alemma:maxoflog}).} Hence, we take the smallest representative on the larger side $L$ and the largest representative on the smaller side $L_a$. For the induction step we apply a case distinction on the operator in $M = \oplus(\ldots)$ with $\oplus\in\{\times,\rightarrow,\wedge\}$. In all cases, if $v$ and $u$ both occur in the same child of $M$, the statement follows by the induction hypothesis (IH). We exploit the semantical indifference of $M$'s children order for $\times$- and $\wedge$-nodes and the symmetry of the $+$ and $\parallel$ order relations to decompose $M_a = \times(\ldots)$ into children that are not abstracted and the two children that are abstracted. Through the (IH) and the language-preserving reduction of process trees into their \emph{normal form} \cite{leemans_robust_2022}, we prove the statement for decomposed $\times$- and $\wedge$-nodes. In case of a $\rightarrow$-node, we derive two different tree structures for $M_a$ that are shown to have fewer traces and events through a combinatorial proof given the respective equations in \autoref{alg:minimal}.  

On the grounds of \autoref{lemma:large}, we establish the correctness of the synchronized EA technique in \autoref{sec:correct}. Hence, we prove the main synchronization result by correctness of the EA technique in \autoref{sec:theorem}.

\section{Synchronization Proof}
\label{sec:proof}

We prove the main synchronization result in \autoref{sec:theorem} by showing that our EA technique produces df-complete logs that enable correct model rediscovery in \autoref{sec:correct}. For full proofs, we refer to the respective lemma in the appendix \autoref{app:lemma}. 



\subsection{Correctness of the EA Technique}
\label{sec:correct}

In \autoref{sec:well}, we have established that the mdf-complete event log $L_a = \Lm(M_a)$ has strictly less traces and is strictly smaller than any event log $L$. Thus, a well-defined EA technique exists. As $\ea_{bpa}$ consists of two steps $\ea_{bpa}^1$ and $\ea_{bpa}^2$, we show that these two steps are correctly specified, i.e., that $\ea_{bpa}$ is a well-defined EA technique that returns df-complete event logs $L_r = \ea_{bpa}(L)$ for the abstracted process tree $M_a = \ma_{bpa}(\disc_{IM}(L))$: $L_a \subseteq L_r$. To that end, we define the transposition equivalence, because $\ea_{bpa}^2$ is specified on the grounds of this equivalence relation.

\begin{definition}[Transposition equivalence]
\label{def:equivalence}
    Let $\sigma_1, \sigma_2 \in \A^*$ be two traces and $\delta_{kendall}$ \cite{kendall_distance_2020} be the minimum number of transpositions required to be applied to $\sigma_1$ yielding $\sigma_1\pr$ such that $\sigma_1\pr = \sigma_2$, if both $\sigma_1$ and $\sigma_2$ are permutations of the multiset of activities $\operatorname{bag}(\sigma_1) = \operatorname{bag}(\sigma_2)$ with $\operatorname{bag}(\sigma) = [\sigma[i] \mid i \in \{1, \ldots, |\sigma| \}]$, and undefined $\delta_{kendall} = \bot$ otherwise. The transposition relation $\sim \;\in \A^* \times \A^*$ is defined by $$\sigma_1 \sim \sigma_2 \;\;\text{iff}\;\; \delta_{kendall}(\sigma_1, \sigma_2) \neq \bot. $$
\end{definition}

Trivially, $\sim$ is an equivalence relation. Because $\ea_{bpa}^2$ searches for a matching between equivalence classes of $L_{tmp} = \ea_{bpa}^1(L)$'s quotient set modulo $\sim$ and equivalence classes of $L_a$'s quotient set modulo $\sim$, we formalize what the matching is.

\begin{definition}[Matching]
\label{def:match}
    Let $L_1, L_2$ be two event logs. Given the transposition equivalence $\sim$, a matching between the two quotient sets $L_1 /\!\sim$ and $L_2 /\!\sim$ is a bijective function $\operatorname{matching}: L_1 /\!\sim \;\rightarrow\; L_2 /\!\sim$ such that:
    \begin{itemize}
        \item for every equivalence class $L_1^{class} \in L_1 /\!\sim$ and corresponding equivalence class $ \operatorname{matching}(L_1^{class}) = L_2^{class}$ it holds: $\exists \sigma_1 \in L_1^{class}, \sigma_2 \in L_2^{class}$ with $ \sigma_1 \sim \sigma_2$, i.e., the respective traces in matched equivalence classes are not distinguishable modulo $\sim$, i.e., not distinguishable modulo transposition. 
    \end{itemize}
\end{definition}

It does not matter whether we define the equivalence of traces in matched pairs of equivalence classes by existence or by universal quantification, as from existence and equivalence the universal quantification directly follows. We use the existence in \autoref{def:match} to emphasize that we only need to check equivalence once for two traces to decide equivalence. Observe that the condition for pairs of equivalence classes in the matching is checked in the if-condition of $\ea_{bpa}^2$ (line 7 \autoref{alg:ea2}). In the following, we prove that a matching between the two quotient sets $\Ltr$ and $\Lar$ exists and that the even split of traces from equivalence classes in the former can always be assigned to traces from equivalence classes in the latter quotient set.

\setcounter{lemma}{7} 

\begin{lemma}[Matching of quotient sets]
\label{lemma:matching}
    If $L$ is restricted and $\ma_{bpa}$ applicable to $M = \disc_{IM}(L)$, there exists a matching between the quotient set $ L_{tmp}/\!\sim$ of event log $L_{tmp} = \ea^1_{bpa}(L,\disc_{IM},\ma_{bpa})$ and the quotient set $L_a/\!\sim$ of $L_a = \Lm(M_a) $ for $M_a = \ma_{bpa}(\disc_{IM}(L))$, i.e., $\operatorname{matching} : L_1 /\!\sim \;\rightarrow\; L_2 /\!\sim$ exists. Also, for every matched pair of equivalence classes $L_{tmp}^{class} \in L_{tmp}/\!\sim$ and $L_{a}^{class} = \operatorname{matching}(L_{tmp}^{class})$ it holds that $L_{tmp}^{class}$ has equal to or more traces than $L_{a}^{class}$: $ |L_{a}^{class}| \leq  | L_{tmp}^{class}|  $ (ii).
\end{lemma}

\rev{The proof strategy is to establish (i) the existence of a matching between the quotient set $L_{tmp}/\!\sim $ of $L_{tmp}$ and the quotient set $L_{a}/\!\sim $ of the reference event log $L_a$ modulo transposition, and (ii) that for any matched pair of equivalence classes $L_{a}^{class} $ and $L_{tmp}^{class}$ from $L_{a}/\!\sim $ and $L_{tmp}/\!\sim $ respectively it holds: $| L_{a}^{class} | \leq |L_{tmp}^{class}|$. We prove (i) by contradiction through the restrictions on an event log (cf. \autoref{def:restricted}), structural properties on abstracted process trees $M_a$ (e.g., no loops other than the self-loop), and by code inspection of $\ea_{bpa}^1$. Hence, the search for a matching in line 5-7 of \autoref{alg:ea2} is correctly specified.
We prove (ii) by (i), the minimality of $L_a$, and \autoref{lemma:large}.}

It follows that the transposition with even splits (line 8-13) returns $L_r$ such that $L_a \subseteq L_r$:

\setcounter{lemma}{1} 

\begin{lemma}[$\ea_{bpa}$ returns df-complete logs]
\label{lemma:df}
    If $L$ is restricted and $\ma_{bpa}$ applicable to $M = \disc_{IM}(L)$, the event log $L_r = \ea_{bpa}(L,\disc_{IM},\ma_{bpa})$ is a df-complete event log for $M_a = \ma_{bpa}(\disc_{IM}(L)).$
\end{lemma}
\begin{proof}
    From \autoref{lemma:matching} and the applied transpositions in line 12 to satisfy every directly-follows relation in $L_a$ by matched sets of traces from their respective equivalence class (line 10 and line 7), it follows that for every trace $\sigma \in L_a$ there exists a trace $\sigma_{abs} \in L_r$ such that $\sigma = \sigma_{abs}$.
\end{proof}

Hence, we can prove the synchronizability as required by our approach.
\subsection{Main Synchronization Theorem}
\label{sec:theorem}

Given the correctness of the EA technique $\ea_{bpa}$ and the isomorphic rediscoverability of abstracted process trees $M_a$, we prove the main synchronization result.

\begin{theorem}[IM and $\ma_{bpa}$ are synchronizable]
\label{thm:sync}
    $\;$\\If $L$ is restricted and $\ma_{bpa}$ applicable to $M = \disc_{IM}(L)$, then $\disc_{IM}(L_a) $ discovered from $L_a = \ea_{bpa}(L,\disc_{IM},\ma_{bpa})$ is isomorphic to $ \ma_{bpa}(\disc(L))$.
\end{theorem}
\begin{proof}
    From \autoref{lemma:df}, it follows that event log $L_a$ is a df-complete event log for $M_a = \ma_{bpa}(\disc_{IM}(L))$. 
    Since $M_a$ is isomorphic rediscoverable (cf. \autoref{lemma:redisc2}), it follows that: $\disc_{IM}(L_a)  \cong M_a$.
\end{proof}

\rev{Therefore, we can abstract discovered process model $M$ through non-order-preserving $\ma_{bpa}$ and synchronously abstract $L$ through $\ea_{bpa}$ such that both $M_a$ and $L_a$ are available for further process intelligence tasks. In the next section, we revisit the illustrative example in \autoref{fig:illustrative} to show what impact the main synchronization theorem has.}

\section{Demonstration on the Illustrative Example}
\label{sec:illustrative}

To put the theoretical results of synchronization into perspective, we demonstrate their impact on the illustrative scenario that we have introduced in \autoref{sec:intro} and \autoref{fig:illustrative}. To that end, we show how the event log $L$ (cf. \autoref{fig:illustrative} \circled{1}) that contains traces of the bank's trading process for two different products (derivative and fixed income) is abstracted by $\ea_{bpa}$ in \autoref{sec:demoea}. Recall that the bank already applied the BPA model abstraction $\ma_{bpa}$ on the discovered process model $M$, as it expects the result $M_a$ to have a more suitable granularity for analyzing the common trading process than $M$. Subsequently, we delineate how the abstracted event log $L_a = \ea_{bpa}(L)$ and $M_a$ help the bank in analyzing the common trading process through process intelligence in \autoref{sec:impact}.

\subsection{Abstracting the Event Log}
\label{sec:demoea}

To begin with, \rev{BPA $\ma_{bpa}$ with $A_{new} = \{\texttt{RQ}, \texttt{OT}, \texttt{N}, \texttt{CT}\}$, $w_t = \wtop = 5/9$, and $\agg$ as depicted in \autoref{fig:illustrative} \circled{3} is applicable to $M$.} 
We illustrate $\ea_{bpa}^1$ (cf. \autoref{alg:ea1}) with the illustrative example. Let $L$ be an event log such that IM discovers $M$ as depicted. For simplicity, we assume that $L$ is a minimal df-complete event log.
\rev{Due to irrelevance of further event attributes for proofs (cf. \autoref{sec:theory}-\autoref{sec:proof}), we only represent activities in traces, but discuss how further attributes can be added. For example $\sigma_1 = \langle \texttt{OLS}, \texttt{TLS}, \texttt{TOS}, \texttt{C}, \texttt{T}, \texttt{T}, \texttt{C},  \texttt{TCS}\rangle$ as denoted by ``Tr.'' in  \autoref{fig:illustrative} \circled{1} is a trace. To start, we have $L = [\sigma_1, \sigma_2, \ldots, \sigma_9]$, $|L| = 9$, and $\norm{L} = 6 * 8 + 5 + 2 * 2 = 57$. Because in a minimal df-complete event log each activity in a self-loop is repeated once, $L$ contains the six interleavings of $\langle\texttt{T}, \texttt{T} \rangle$ and $\langle \texttt{C}, \texttt{C} \rangle$ in $\sigma_1, \sigma_3, \ldots, \sigma_7$ plus $\sigma_2$ and plus the two traces $\sigma_8 = \langle \texttt{GO},\texttt{RO} \rangle$ and $\sigma_9 = \langle \texttt{OLS}, \texttt{ODS} \rangle$. Given $A_{new} = A_{M_a}$ (cf. \autoref{fig:illustrative} \circled{5} and \autoref{def:bpa}), $\ea^1_{bpa}$ always computes $A_{\ma} = \{\sigma[1], \ldots, \sigma[|\sigma|]\}$, i.e., all events are abstracted (line 4). Thus, $\sigma_{1}[1]$ is the first event abstracted by $\texttt{RQ} \in A_{new} $ (line 8) and the condition in line 9 is always true, because there exists no concrete event $v$.}

\rev{Next, for the two events $\texttt{OLS}$ and $\texttt{TLS}$, the single abstract event $\texttt{OT}$ is added due to the first appearance condition in line 8. The abstract event $\texttt{N}$ for the four events $\texttt{C}, \texttt{T}, \texttt{T}, \texttt{C}$ is repeated, because $\texttt{N} \parallel_{M_a} \texttt{N}$ holds and triggers adding a self-loop (cf. \autoref{sec:bpa}) that can be rediscovered by repeating $\texttt{N}$ once. Altogether, we have $\sigma_{abs1} = \langle \texttt{RQ}, \texttt{OT}, \texttt{N}, \texttt{N}, \texttt{CT}\rangle \in L_a$ (cf. \autoref{fig:illustrative} \circled{4}) and analogously for the remaining eight traces resulting in $L_{tmp} = [\sigma_{abs1}^7, \sigma_{abs8}^2]$ .}

\rev{Before we proceed with line 16, we point out how further attributes can be added. In line 10, new abstract events $x$ are created. Here, any attribute of the current event $e$ or any aggregation function over the set of $x$'s concrete events $A_{\ma}(x) = \{e \in \sigma\mid e \in \agg(x)\}$ can be copied or computed. For instance, we have applied a \emph{last attribute aggregation} with \emph{even time split} for repeated $x$ on every attribute in \autoref{fig:illustrative} \circled{4}: Because $A_{\ma}(\texttt{N}) = \{\texttt{T}, \texttt{C}\}$, the four corresponding events $e$'s attributes in $\sigma_1$ are taken together, evenly split\footnote{Ties can be decided towards the first abstract event.} by respecting time, and the respective last attribute is assigned to $\texttt{N}$. The result shows that the event with ID ``a3'' \autoref{fig:illustrative} \circled{4} has timestamp $t_5$ and the ``Terms'' value of the event with ID ``5'' in \autoref{fig:illustrative} \circled{1}. Yet, proposing and evaluating different attribute aggregations goes beyond the scope of this paper. Still, we advocate adding the ``concrete'' event attribute to every abstract event for storing the respective $A_{\ma}(x)$ and, thus, maintaining flexibility.}

\rev{Because $\texttt{DQ} +_{M_a} y \in p_{M_a}$ (cf. the initial XOR-gateway in \autoref{fig:illustrative} \circled{5}) for any other activity $y \in C :=  A_{M_a} \setminus\{\texttt{RQ}\}$, it holds that $A_{\times} = \{ C \}$ in line 17, i.e., module $C$ is the only XOR-complete module in $MDT(G)$. Nevertheless, the corresponding concrete events $\texttt{RO}$ and $ \texttt{DCS}$ never occur together with any concrete event $e \in A_{\ma}(y)$ of other abstract activities $y \in C$ in any trace $\sigma \in L$, because they were in choice relation already. To illustrate, for $\sigma_2$, we have $y = \texttt{OT}, A_{\ma}(y) = \{\texttt{OW}\}$ and $\texttt{OW}$ never occurs with $\texttt{RO}$ or \texttt{ODS}  (cf. \autoref{fig:illustrative} \circled{3}). Hence, no events are deleted in line 17 and $\ea^1_{bpa}$ returns $L_{tmp} = [\sigma_{abs1}^7, \sigma_{abs8}^2]$ that contains 9 traces.}

\rev{To continue the illustration, $\Lm$ computes $L_a = [\sigma_{abs1}, \sigma_{abs8}]$ with $\sigma_{abs8} = \la \texttt{RQ},\texttt{DQ}\rangle$ (line 1) such that $L_a/\!\sim = \{ [\sigma_{abs1}] , [\sigma_{abs8}]\}$ is the quotient set of $L_a$ by $\delta_{kendall}$ (line 3). Likewise, $L_{tmp}/\!\sim =\{[\sigma_{abs1}^7], [\sigma_{abs8}^2]\}$ is the quotient set of $L_{tmp}$ (line 4). Next, we iterate over equivalence classes of both quotient sets (line 5-6). First, $L_a^{class} = [ \sigma_{abs1}]$ and $L_{tmp}^{class} = [ \sigma_{abs1}^7]$ share traces that are permutations of the same multiset of activities, i.e., the condition in line 7 is true. Because $|L_{tmp}^{class}| / |L_{a}^{class}| = 7 / 1 = 7$ is an integer quotient, no remainder is left to be evenly spread across the $n_1, \ldots , n_{|L_{a}^{class}|}$ sizes of splits, i.e., $n_1 = 7$ (line 8). Given the split sizes, the $j$th trace $\sigma \in L_{a}^{class}$ is the reference for the $n_j$ closest traces in $ L_{tmp}^{class}$ in terms of number of transpositions ($\delta_{kendall}$) (line 10).} 

\rev{Due to $n_1 = 7$ and the single trace in $L_{a}^{class}$, the $\operatorname{closestTraces}$ operator assigns all seven traces in $ L_{tmp}^{class}$ with a distance of 0 transpositions to $L_{tmp}^{n_1} $. Consequently, $L_{tmp}^{class}$ is empty afterwards (line 11), no transpositions must be applied to traces $L_{tmp}^{transposed} = L_{tmp}^{n_1} $ (line 12), and the seven traces in $L_{tmp}^{n_1} $ are added to the result $L_r$ (line 13): $L_r = [\sigma_{abs1}^7]$. In \autoref{fig:illustrative} \circled{4}, the first two traces of the $\sigma_{abs1}$ trace variant are depicted. Analogously, the second for-loop iteration (line 5) yields $L_r = [\sigma_{abs1}^7, \sigma_{abs8}^2] $, i.e., no transpositions were necessary overall.}

To sum up, the abstracted event log $L_r$ maintains the multiplicities of the concrete event log $L$ and the respective distributions of values in the additional event attributes. In particular, the abstracted event log $L_r$ is equivalently related to $M_a$ as the concrete $L$ was related to $M$. In the next section, we demonstrate the impact of having both $L_a$ and $M_a$ for subsequent process intelligence.

\subsection{Impact on Process Intelligence}
\label{sec:impact}

Coming back to the bank scenario in \autoref{fig:illustrative}, correct synchronized abstraction of both $M$ and $L$ into $M_a$ (cf. \autoref{fig:illustrative} \circled{5}) and $L_a$ (cf. \autoref{fig:illustrative} \circled{4}) allows the bank to analyze the common trading process of both products simultaneously with, e.g., process enhancement (cf. \autoref{fig:illustrative} \circled{8}). To that end, $L_a$ allows to enhance $M_a$ with the frequency information that seven trades negotiated the terms in particular wrt. the ``spread'' as recorded in basis points (bp). Here, the bank is most interested in the question: ``Does the client and requested product affect the final terms after renegotiation?''. 
To answer the question, the bank must define a classification problem that relates the \texttt{OT}'s event attribute value for attributes ``Client'' and ``Product'' to the last event attribute value of the second \texttt{N} event in each trace that has the \texttt{CT} event as the end event. Because the synchronized EA technique $\ea_{bpa}$ enables to automatically compute the abstracted $L_a$ that contains the seven trades corresponding to the abstracted process model $M_a$, the bank can define the classification problem given $L_a$ and $M_a$. 

Having only the abstracted process model $M_a$, the bank is neither able to enhance the abstracted process model with further perspectives for, e.g., presenting different aspects of the process to different stakeholders in the bank, nor is it able to define the classification problem, as it would lack the abstracted event log $L_a$. Moreover, it can start predicting and simulating the common trading process by learning both a simulation and prediction model through $L_a$ and $M_a$. To sum up, the bank is able to flexibly apply a similar set of process intelligence tasks on an abstracted process model and abstracted event log as it was used to on the concrete event log and process model.

\section{Conclusion}
\label{sec:conclusion}

We propose a novel synchronization approach that closes the gap between MA and EA techniques.
\rev{Consequently, we can apply MA technique $\ma$ on discovered process models $M$ and compute the corresponding abstracted event log $L_a$ through the synchronized EA technique $\ea_{\ma}$.}
\rev{The proposed approach is the first to formalize the impact of EA techniques on the discovered process model and the first to enable process intelligence tasks grounded in the real-world behavior contained in $L_a$.}
So far, our approach is limited to the IM and its relatively simple event log conceptualization. 
Furthermore, our approach focuses on size as an indicator of complexity 
\rev{and only briefly discusses event attribute abstractions that go beyond the control-flow, but are driven by the data flow.}
In future work, we will investigate 
extended conceptualization of event logs.
Second, we will extend the set of process discovery techniques and integrate optimization frameworks for MA that capture further complexity metrics. Third, we will parametrize synchronized EA techniques to allow for even more faithful abstraction of event logs.
\newline

\appendix

\section*{Appendix}

The appendix includes \autoref{lemma:large} and \autoref{lemma:df} from the main paper with full proofs. Moreover, the appendix includes the additional \autoref{lemma:redisc}, \autoref{lemma:redisc2}, \autoref{lemma:gensize}, \autoref{lemma:maxoflog},  and \autoref{lemma:matching} from this extended version with full proofs. In addition, we report a simple helper lemma in \autoref{alemma:semid}.

\section{Lemmata With Full Proofs}
\label{app:lemma}


\setcounter{lemma}{2} 

\begin{lemma}[Process trees in $C_{a}$ are isomorphic rediscoverable]
\label{alemma:redisc}
\srm{Let $M$ be a process tree and $L$ be an event log. }
If $M$ is in class $C_{a}$ and $M$ and $L$ are df-complete, i.e., $M \dfc L$, then $\disc_{IM}$ discovers a process tree $M\pr$ from $L$ that is isomorphic to $M$. 
\end{lemma}
\begin{proof}
    $M$ meets all model restrictions for the class of language-rediscoverable process trees in \cite{leemans_tech_2013} except for the self-loop $\tloop(v, \tau)$ with $v\in A$ that can occur as a (sub-)tree in $M$. The self-loop violates both the restriction on disjoint start and end activities for the first branch of the loop (cf. restriction 2 in \cite{leemans_tech_2013}) and the restriction that no $\tau$'s are allowed in $M$ (cf. restriction 3 in \cite{leemans_tech_2013}). 
    Hence, we distinguish whether $M$ contains a self-loop. For both cases, we build on the proof of Theorem 14 \cite{leemans_tech_2013}, because Theorem 14 is proven through showing isomorphic rediscoverability by induction on the size of $M$. 

    \noindent\textbf{Case no self-loop}: Because $M$ meets all model restrictions for the class of language-rediscoverable process trees, the conditions of Theorem 14 in \cite{leemans_tech_2013} are met and $M$ isomorphic rediscoverable.

    \noindent\textbf{Case self-loop}: 
    Because the self-loop does not contain nested subtrees, it suffices to extend Theorem 14's induction on the size of $M$ by an additional base case: $\tloop(v, \tau) $ with $v \in A$. The additional base case holds, because the ``Strict Tau Loop'' fall through rediscovers the self-loop from any splitted log $L_v \subseteq [\la v,v \ra^m, \la v,v,v\ra^n, \ldots]$ that occurs during recursion of IM. 
\end{proof}


\begin{lemma}[$\ma_{bpa}$ abstracted process trees are in $C_a$]
\label{alemma:redisc2}
    If $\ma_{bpa}$ is applicable to process tree M, then $M_a = \ma_{bpa}(M) $ is in class $C_a$.
    
    \end{lemma}
\begin{proof}
    The abstracted process tree $M_a$ does not contain duplicate activities, because $M$ does not contain duplicate activities (condition 1 in \autoref{def:cbpa}) and new activities $x \in A_{new}$ (cf. condition 4) are only added once to the process tree $M_a$ during synthesis (cf. Algorithm 3.2 in \cite{smirnov_businessm_2012}). 
    
    \noindent A $\circlearrowleft\,$-node $M_{\circlearrowleft} = \:\circlearrowleft(M_1, \ldots, M_n)$ and a $\wedge$-node $M_{\wedge} = \wedge(M_1\pr, \ldots, M_m\pr)$ in a process tree $M$ are indistinguishable through the BP $p_M$ (i.e., $x_i \parallel y_j \in p_M $ for all $x_i, y_j\in A_{M_{\circlearrowleft}}$ where $x_i$ is an activity in child $M_i$ of $M_{\circlearrowleft}$ and $y_j$ is an activity from another child $M_j$, $i \neq j$ of $M_{\circlearrowleft}$ and similarly $x_i\pr \parallel y_j\pr \in p_M$ for two activities from different children of $M_{\wedge}$), because the order relations in the BP $p_M$ are defined through the eventually-follows relation. 
    For $\ma_{bpa}$, there exists a choice in the synthesis step of $\ma_{bpa}$: Either construct a loop node for an AND-complete module in the modular decomposition tree or construct a parallel node, since there is no information in the abstracted BP $p_{M_a}$ that allows to differentiate these two nodes. $\ma_{bpa}$ chooses the parallel node and always constructs a parallel node for AND-complete modules with the exception of singleton modules $x \parallel x \in p_{M_a}$ in which case it constructs the self-loop $\tloop(x, \tau)$ (cf. line 7-8 and 15-16 in Algorithm 3.2 in \cite{smirnov_businessm_2012}). 

    \noindent Modules of the modular decomposition tree $MDT(G)$ of the graph $G(p_{M_a})$ can by definition \cite{smirnov_businessm_2012} only contain activities.
\end{proof}

\begin{lemma}[Number of traces and size of $\gen(M).\text{log}$]
\label{alemma:gensize}
    If $M$ is in $C_c$, then $|\Lm(M)|$ is computed by $\gen$, $|\Lm(M)| = |\emph{lens}(M)|$, and the size is $\norm{\Lm(M)} = \sum_{k=1}^{|\Lm(M)|} \emph{lens}(M)[k] $. 
\end{lemma}
\begin{proof}
    We sketch the structural induction proof for the number of traces and size by case distinction on the process tree's node operator.

    \noindent $M = \times(M_1, \ldots, M_n)$: The union of logs $L = \bigcup_{i=1}^n \Lm(M_i)$ constructs a log $L$ whose number of traces equals the sum of trace numbers and whose trace lengths equal the concatenation of all children's $M_i$ sequence of trace lengths $\text{lens}(M_i)$. The length of concatenating $n$ sequences of length $\text{lens}(M_i)$ equals the sum of $n$ lengths $|\text{lens}(M)| = |\Lm(M)|$. 
    
    \noindent $M = \:\rightarrow(M_1, \ldots, M_n)$: The set of all sequential concatenations $ L = \{\sigma_1\cdot \sigma_2 \cdot \ldots \cdot \sigma_n \mid \forall i \in \{1\ldots n\} : \sigma_i \in \Lm(M_i)\}$ constructs a log $L$ whose number of traces equal the number of ordered pairs in the cartesian product $ \bigtimes_{i=1}^n \{1, \ldots, \\|\Lm(M_i)|\}$. Hence, $|\Lm(M)| = \prod_{i=1}^{n} |\Lm(M_i)|$. The bijection $\iota$ enumerates the ordered pairs for summation. 
    
    \noindent $M = \wedge(M_1, \ldots, M_n)$: Interleaving of multiple event logs $\Lm(M_i)$ 
     through a $\wedge$-node requires interleaving of $ \prod_{i=1}^{n} |\Lm(M_i)|$ different ordered pairs of traces we refer to as combinations. 
    As each of these trace combinations can have traces of varying lengths, the function $\gen$ enumerates the respective combinations of traces lengths through the bijection $\iota$ and sums the respective number of interleaving the trace combinations. Each combination of traces $(\sigma_1, \ldots \sigma_n) \in \Lm(M_1) \times \ldots \times \Lm(M_n)$ to be interleaved is enumerated by index $k$ in the domain of $\iota$: $k \in \operatorname{dom}(\iota)$. Hence, the $k$th enumerated trace combination has its corresponding sequence of trace lengths $\bigodot_{i=1}^{n}  \la \text{lens}(M_i)[\iota_{k,i}] \ra$ (line 12 in Algorithm 3 of the main paper). The number of interleavings of two traces $\sigma_1$ and $\sigma_2$ is $\binom{m}{\,\la l_1,\; l_2\ra\,} = \frac{(l_1+l_2)!}{l_1! l_2!} $ with $|\sigma_1| = l_1 $ and $|\sigma_2| = l_2$, because interleaving two sequences without changing their respective order is equivalent to shuffling two card decks without changing the card order of the two decks \cite{aldous_shuffling_1986}. 
    As the number of riffle shuffle permutations equals $\frac{(p+q)!}{p! q!} $ for $p$ the number of cards in the first deck and $q$ the number of cards in the second deck \cite{aldous_shuffling_1986}, the number of traces $\sigma \in \sigma_1 \diamond\sigma_2 $ equals $\frac{(k_1+k_2)!}{k_1! k_2!} $ with $|\sigma_1| = k_1 $ and $|\sigma_2| = k_2$. 
    Generalizing the number of two trace interleavings to the number of $n$ traces interleaved yields $\binom{m_1}{\,\la l_1,\; l_2\ra\,} * \binom{m_2}{\,\la l_1 + l_2,\; l_3\ra\,} * \ldots *\binom{m_n}{\,\la \sum_{i=1}^{n-1}l_i,\; l_n\ra\,}  = \frac{(l_1+l_2)!}{l_1! l_2!} * \frac{(l_1+l_2 + l_3)!}{l_3! (l_1 + l_2)!  } *\frac{(l_1+l_2 + l_3 + l_4)!}{l_4! (l_1 + l_2 + l_3)!} * \ldots * \frac{(\sum_{i=1}^{n} l_i)!}{l_n! (\sum_{i=1}^{n-1} l_i)!} = \frac{(\sum_{j=1}^{n} l_j)!}{\prod_{j = 1}^{n}l_j!} = \genfrac{(}{)}{0pt}{1}{m}{\la l_1, \ldots, l_n\ra}$, because after two traces have been interleaved, we can take the already interleaved trace as a new trace for the next trace to be interleaved with. Hence, the number of $n$ traces interleaved equals the multinomial coefficient for the $n$ different trace lengths $l_1, \ldots , l_n$. Consequently, for each $k$th combination of traces in the $\gen$ function at line 12, the multinomial coefficient is computed for the corresponding sequence of trace lengths $ \la l_{k,1}, \ldots, l_{k,n}\ra = \bigodot_{i=1}^{n} \la \text{lens}(M_i)[\iota_{k,i}]\ra$. Each of these interleaved traces are summed together to yield the number of traces.

    \noindent Since there are $ \prod_{j=1}^{n} |\Lm(M_j)|$ different trace combinations to be interleaved, there can only be $ \prod_{j=1}^{n}|\Lm(M_j)|$ different lengths of traces as a result of interleaving. The interleaved trace length for the $k$th trace combination equals $m_k = \sum_{i=1}^n \text{lens}(M_i)[\iota_{k,i}] = \sum_{i=1}^n l_{k,i} $, which has to be repeated in $\text{lens}(M)$ for the number of interleavings: $\la m_k\ra^{\mnomc{\bigodot_{i=1}^n \la \text{lens}(M_i)[\iota_{k,i}]\ra}{m_k}}$. 
\end{proof}

\setcounter{lemma}{6}

\begin{lemma}[Maximal $\Lm(M_a)$]
\label{alemma:maxoflog}
    If $L$ is restricted and $\ma_{bpa}$ applicable to $M = \disc_{IM}(L)$, $w_t = \wtop$, $\agg(x) = \{v,u\} $, and $ \agg(z) = \{z\} $ with $ z\in A\setminus\{v,u\}$, i.e., $\ma_{bpa}$ aggregates exactly two activities into $x$, then $M_{a} = \ma_{bpa}(M)$ has the most traces and maximal size of all $\Lm(M_a\pr)$ generated for other $M_a\pr$ with $A_c\pr \cap \{v,u\} = \emptyset $ and $|\bigcup_{y\in A_{new}\pr} \agg\pr(y)| > |\bigcup_{y\in A_{new}} \agg(y)|$ (cf. Definition 2 in our main paper).
\end{lemma}
\begin{proof} Let $M_a$ be the abstracted process model for $w_t =\wtop, \agg(x) = \{v,u\}$ and $M_a\pr$ be the abstracted process model for $w_t \leq \wtop$ and an aggregate function that abstracts more than two concrete activities, i.e., $|\Aa\pr \cap \AM| < |\AM| -2$. 
From condition (4) Definition 2 in our main paper, it follows that $|A_{M_a\pr}| < |A_{M_a}|$, because either are at least three concrete activities abstracted into a single abstract activity in $M_a\pr$ or $m$ abstract activities aggregate at least $m+1$ concrete activities. Thus, $M_a\pr$ cannot have more traces and events in $\Lm(M_a\pr)$ than $\Lm(M_a)$ as a result of more activities. Hence, $M_a\pr$ must have a ``different'' tree structure in terms of node operators and their children. 

\noindent The abstract process tree is constructed given the modular decomposition (MDT) of the ordering relations graph $G(p_{M_a\pr}) $ and $G(p_{M_a}) $ such that structural tree differences must be caused by differences in the the behavioral profiles $p_{M_a\pr} $ and $p_{M_a}$. The two behavioral profiles are different 
\begin{enumerate}
    \item with respect to their sizes $|p_{M_a\pr}| < |p_{M_a}|$,
    \item with respect to order relations that involve activities $q \in A_c$, i.e., activities $q$ are not abstracted in $M_a$ but are abstracted in $M_a\pr$,
    \item with respect to order relations of abstract activities $y\pr \in A_{new}\pr$ to $q\pr \in \Aa\pr$ (cf. condition 3 Definition 2 in our main paper) and $y \in A_{new}$ to $q \in \Aa$.
\end{enumerate}
   The first two differences imply more traces and events in $\Lm(M_a)$. Hence, a larger number of traces and events in $\Lm(M_a\pr)$ can only be the result of different order relations between abstract activities $y\pr$ to $q\pr$ and $y$ to $q$. From \autoref{alemma:sort}, more traces and events can only occur, if order relations of abstract activities $y\pr$ to a $q\pr$ are less restrictive than from $y$ to $q$.

\noindent By definition, $L$ is restricted and $w_t = \wtop$. Hence, $M$ is in $C_c$ and meets the requirement on the model structure that excludes children of less than two traces for any node $\rightarrow(\ldots)$ in $M$. From \autoref{alemma:sort}, it follows that the order of returning order relations in , i.e., first $+_{M_a}$, then $\rightsquigarrow^{-1}_{M_a}$, then $\rightsquigarrow_{M_a}$, then $\parallel_{M_a}$ (cf. \ref{order}), aligns with the order of process tree operators for a node $M_{\oplus}$ whose children remain the same, but whose root node changes. By code inspection of Algorithm 3.1 in \cite{smirnov_businessm_2012} (that corresponds to $\operatorname{dv}_{\agg, w_t}$ in our main paper), an order relation $x \diamond_{M_a} q$ is only returned, if $ v\diamond_M q$ for $v\in\agg(x)$ and $ q\in \AM$ occurs either relatively more often in the behavioral profile $p_M$ or, in case of equal relative frequencies, the more restrictive relation according to the order \ref{order} is selected. 

\noindent In any case, $x \diamond_{M_a} q$ is only returned, if the respective order relation is either already the most prevalent in $M$, i.e., the most prevalent in the behavioral profile $p_M$ of $M$, or, in case of conflicts between order relation frequencies, it is the more ``restrictive'' according to \ref{order}. Hence, Algorithm 3.1 in \cite{smirnov_businessm_2012} chooses an order relation between two activities that corresponds to fewer traces and fewer events in $\Lm(M_a\pr)$ over an order relation between two activities that corresponds to more traces and more events. Therefore, order relations between abstract activities $y\pr$ and $q\pr$ cannot result in more traces and events in $\Lm(M_a\pr)$. Also, setting $w_t < \wtop$ can only further decrease the number of traces and events in $\Lm(M_a\pr)$.

\noindent Altogether, all three differences between $p_{M_a\pr}$ and $p_{M_a}$ imply that $\Lm(M_a\pr) \leq \Lm(M_a)$.

\end{proof}

\setcounter{lemma}{0} 

\begin{lemma}[$\Lm(M_a)$ is smaller]
\label{alemma:large}
    If $L$ is restricted and $\ma_{bpa}$ applicable to $M = \disc_{IM}(L)$, then $L_{a} = \Lm(M_a)$ 
    has fewer traces and is smaller than $L$: $|L_{a}|< |L|$ and $\norm{L_{a} }< \norm{L}$.
\end{lemma}
\begin{proof}
    From \autoref{alemma:semid} it follows that we can generate mdf-complete event logs $\Lm(M)$ for $M$.
    From \autoref{alemma:gensize}, it follows that $\Lm(M)$ is the smallest event log for which IM discovers the same process tree $M$, i.e., it is the smallest representative of all restricted event logs: $|\Lm(M)| \leq |L\pr|$ and $\norm{\Lm(M)} \leq \norm{(L\pr)}$ for every $L\pr \in [L]_M = \{L\pr \subseteq \E^* \mid \disc_{IM}(L\pr) \cong M \wedge L\pr\: \text{is restricted}\}$. Additionally, from \autoref{alemma:maxoflog} it follows that it suffices to only consider $\ma_{bpa}$ with $w_t = \wtop$ that abstracts two concrete activities: $\agg(x) = \{v,u\}$ and $\agg(z) = \{z\} $ for $v,u \in \AM$ and $z \in \AM \setminus \{v,u\}$. Thus, $L_a$ is generated for $M_a = \ma_{bpa}(M)$ in which two concrete activities are abstracted with $w_t=\wtop$ . Consequently, the following induction on the size $|M|$ proves $|L_{a}|< |\Lm(M)|$ and $\norm{L_{a} } < \norm{\Lm(M)}$. 

    \noindent\textbf{Base Cases:} \begin{itemize}
    \item $M = \wedge (a,b)$: $\agg(x) = \{a,b\}$ and $\wtop=0.5$ results in $M_a= \:\tloop(x,\tau)$ with $L_a = \{\la x,x\ra\}$. Hence, $|L_a| = 1 < |\Lm(M)| = 2$ and $\norm{L_a} = 2 < 4 = \norm{\Lm(M)}$.
    \item $M = \times(a,b)$: $\agg(x) = \{a,b\}$ and $ \wtop=1$ results in $M_a= x$ with $L_a = \{\la x\ra\}$. Hence, $|L_a| = 1 < |\Lm(M)| = 2$ and $\norm{L_a} = 1 < 2 = \norm{\Lm(M)}$.
    \item $M = \wedge(a, \tloop(b,\tau))$: $\agg(x) = \{a,b\}$ and $\wtop=0.75$ results in $M_a= x$ with $L_a = \{\la x\ra\}$. Hence, $|L_a| = 1 < |\Lm(M)| = 3$ and $\norm{L_a} = 1 < 9 = \norm{\Lm(M)}$.
    \item $M = \times(a, \tloop(b,\tau))$: $\agg(x) = \{a,b\}$ and $\wtop=0.75$ results in $M_a= \:\tloop(x,\tau)$ with $L_a = \{\la x,x\ra\}$. Hence, $|L_a| = 1 < |\Lm(M)| = 2$ and $\norm{L_a} = 2 < 3 = \norm{\Lm(M)}$.
    
    \end{itemize}

    \noindent\textbf{Induction hypothesis (IH):} For any process tree $M\pr$ of smaller size than $M$ that is discovered from restricted event log $L\pr$ such that $M_a\pr = \ma_{bpa}(M\pr)$ is applicable to $M\pr$ for $w_t\pr = \wtop\pr$, the mdf-complete log  $L_a\pr$ has fewer traces and is smaller than $\Lm(M)\pr$: $|L_a\pr| < |\Lm(M)\pr|$ and $\norm{L_a\pr} < \norm{\Lm(M)\pr}$.

    \noindent\textbf{Induction step:} 
    Let $M = \oplus (M_1, \ldots, M_n)$. The loop operator cannot occur as the root node, as it can only occur in a self-loop $\tloop(v, \tau)$ with $v \in \AL$ in a process tree $M$ discovered from restricted event log $L$ (cf. \autoref{alemma:semid}), i.e., it is either covered in base cases or part of $M$ as a subtree in one of the $M_i$'s. Thus, apply case distinction on the operator node $\oplus \in \{\times, \rightarrow, \wedge\}$:
    \begin{itemize}
        \item Case $\oplus = \times$: 
    \end{itemize}
    
\noindent There exist two cases of how $\agg(x) = \{v,u\}, \agg(z) = \{z\}$ for $z \in \AM \setminus \{v,u\}$ can be defined on $M$:
\begin{itemize}
\renewcommand\labelitemi{--}
        \item Case $v,u \in M_i, i\in \{1, \ldots, n\}:$ Hence, both concrete activities occur in the same child $M_i$ of $M$. By (IH), the inequalities hold.
        \item Case $v \in M_i, u\in M_j, i,j\in \{1, \ldots, n\}, i \neq j:$ Hence, the two concrete activities $v$ and $u$ occur in two different children $M_i$ and $M_j$ of $M$. Because the choice order relation $+$ is symmetric and the children $M_1, \ldots, M_n$ of $M$ can be reordered without changing the language $\Lm(M)$, the abstracted process tree can be decomposed: $M_a\pr = \times(\ma_{bpa}\bigl(\times(M_i, M_j)\bigr), M_{r_1}, \ldots, M_{r_m})$ with $m = n-2 $ and $r_1, \ldots, r_m \in \{1, \ldots, n\} \setminus \{i,j\}$. $M_a\pr$ may not always be in normal form as the process tree $\ma_{bpa}\bigl(\times(M_i, M_j)\bigr)$ may have a choice operator as a root node. Nevertheless, reducing $M_a\pr$ to a normal form with the reduction rules in Definition 5.1 \cite{leemans_robust_2022} yields the abstracted process tree $M_a$ and the reduction rules preserve the language (\namedlabel{reduce}{\circled{r}}), i.e., $\Lm(M_a\pr) = \Lm(M_a)$ such that the number of traces and sizes are equal. Hence, $M_a = \ma_{bpa}(M)$ is decomposable as specified by $M_a\pr$ such that the abstracted process tree only differs to $M$ in the abstraction of $M_i$ and $M_j$. Because $|\times(M_i, M_j)| <|M|$, the inequalities follow from (IH).
\end{itemize}
    \begin{itemize}
        \item Case $\oplus = \:\rightarrow$: 
\end{itemize}
\noindent There exist two cases of how $\agg(x) = \{v,u\}, \agg(z) = \{z\}$ for $z \in \AM \setminus \{v,u\}$ can be defined on $M$:
\begin{itemize}
\renewcommand\labelitemi{--}
        \item Case $v,u \in M_i, i\in \{1, \ldots, n\}:$ Analogous to case $\oplus=\times$.
        \item Case $v \in M_i, u\in M_j, i,j\in \{1, \ldots, n\}, i \neq j:$ Hence, the two concrete activities $v$ and $u$ occur in two different children $M_i$ and $M_j$ of $M$. Without loss of generality, we assume $i<j$. If $i>1 $ or $j <n$, abstraction of $M$ can be decomposed into $M_a = \:\rightarrow(M_1, \ldots, M_{i-1}, \ma_{bpa}(M_{i\leq r\leq j}), M_{j+1}, \ldots, M_n)$ with $M_{i\leq r\leq j} = \:\rightarrow(M_i, M_{i+1}, \ldots, M_j)$, because the behavioral profile $p_{M_a}$ differs to the behavioral profile $p_M$ only for ordering relations of activities $y \in A_{M_{i\leq r\leq j}}$. Since \ref{reduce} and $|M_{i\leq r\leq j}| < |M|$, the inequalities hold by (IH). 
        
        If $i = 1$ and $j = n$, then $M_{i\leq r\leq j} =M$, i.e., $\ma_{bpa}$ abstracts the whole process tree $M$. For any $q \in \bigcup_{r \in \{2, \ldots, n-1\}} A_{M_r}$, it holds that $v \rightsquigarrow_M q$ and $u \rightsquigarrow_M^{-1}q$. Thus, all relative frequencies of ordering relations (line 7-10 in Algorithm 3.1 in \cite{smirnov_businessm_2012}) are equal to $0.5$ such that $+_{M_a}$ is always returned: $\wm(x,q) = w(x +_{M_a}q) = 0.5$. On the one hand, for any $q_1 \in A_{M_1} \setminus \{v\}$, it holds $u  \rightsquigarrow_M^{-1} q_1$. Thus, if $v +_M q_1$ or $v \rightsquigarrow_M q_1$, then $w_{max}(x,q_1) = w(x +_{M_a} q_1) = 0.5$. If $v \rightsquigarrow_M^{-1} q_1$ or $v \parallel_{M}q_1$, then $w_{max}(x,q_1) = w(x \rightsquigarrow_{M_a}^{-1} q_1) \geq 0.5$. On the other hand, for any $q_n \in A_{M_n} \setminus \{u\}$, it holds $v  \rightsquigarrow_M q_n$. Thus, if $u +_M q_n$ or $u \rightsquigarrow_M^{-1} q_n$, then $w_{max}(x,q_n) = w(x +_{M_a} q_n) = 0.5$. If $u \rightsquigarrow_M q_n$ or $u \parallel_{M}q_n$, then $w_{max}(x,q_n) = w(x \rightsquigarrow_{M_a} q_n) \geq 0.5$. Additionally, $w_{max}(x,x) = w(x +_{M_a} x) = 0.5$, i.e., $x$ is not added as a self-loop $\tloop(x, \tau)$ to $M_a$ during the synthesis step of $\ma_{bpa}$.

        Altogether, the abstract activity $x$ is in choice relation $v +_{M_a}z$ to any activity $q$ of $M_2, \ldots, M_{n-1}$, $x$ is either in choice $x +_{M_a}q_1$ or inverse strict order relation $x \rightsquigarrow_{M_a}^{-1} q_1$ to any activity $q_1$ of $M_1$, and $x$ is either in choice $x+_{M_a}q_n$ or strict order relation $x \rightsquigarrow_{M_a} q_n$ to any activity $q_n$ of $M_n$. If $x$ is in choice relation to all activities $q_1$ and $q_n$ in $M_1$ and $M_n$ respectively, then
        \begin{align*}
        M_a = \times(x,\rightarrow(M_1\pr, M_{2}, \ldots, M_n\pr)) && \hspace{0em}\text{(I)} 
        \end{align*}
        \noindent Although $M_2, \ldots, M_{n-1}$ may, in fact, change through $\ma_{bpa}$, if they contain an \emph{optional} node $\times(\ldots,\tau, \ldots)$, we do not consider optional nodes for simplicity. Optional nodes are always abstracted by removing the $\tau$ (if the node has more than two children) or by ``moving up'' the other activity $q \in A_{\times(\ldots,\tau, \ldots)}$ as a direct child to the parent node. In both cases, the number of traces and events of the mdf-complete log $\Lm(M_{r})$ for $r \in \{2,\ldots, n-1\}$ decreases, so the $L_a$ we consider in the following by ignoring optional nodes is at least as large as the $L_a\pr$ that would result from also considering abstraction of optional nodes.
        
        \noindent Let $M_v$ and $M_u$ be the nodes of $M_1$ and $M_n$ in which $v$ and $u$ occur as children respectively. 
        If $M_v = \oplus(v, M_{v,2})$ or $M_u =  \oplus(u, M_{u,2})$, then the other node $M_{v,2}$ and $M_{u,2}$ respectively ``move up'' as children to the parent node of $M_v$ and $M_u$ respectively. If $M_v$ or $M_u$ has more than two children, child $v$ and $u$ is eliminated from $M_v$ and $M_u$ respectively. For all four cases of how $M_1$ and $M_n$ are changed into $M_1\pr$ and $M_n\pr$ through changing $M_v$ and $M_u$ respectively, it holds $|\Lm(M_1\pr)| < |\Lm(M_1)| $ and $\norm{\Lm(M_1\pr)} < \norm{\Lm(M_1)} $ as well as similarly for $M_n$ and $M_n\pr$.
        Thus, for (I), it follows that: 
        \end{itemize}
        \begin{align*}
            |L_a| &= 1 + |\Lm(M_1\pr)| * |\Lm(M_n\pr)| * \prod_{r \in \{2, \ldots, n-1\}} |\Lm(M_r)| < |\Lm(M)|
         \end{align*}
       \begin{itemize}
           \item[] because from line 9 of Algorithm 3 in our main paper, we can compute the number of traces of a sequence as the product of its children's number of traces and two children with less traces always outweigh the new trace $\la v\ra \in L_a$. 
           \item[] The size of $\Lm(M)$ is: 
       \end{itemize}
       \vspace{-1em}
\begin{align*}    
    \norm{\Lm(M)} &=\sum_{j=1}^{|\Lm(M)|} \text{lens}(M)[j]\\
                   &= \sum_{j=1}^{|\Lm(M)|} (\bigodot_{k=1}^{|\Lm(M)|} <\sum_{i = 1}^n \text{lens}(M_i)[\iota_{k,i}]>)[j]\\
                   \shortintertext{Break the trace lengths sequence up and rewrite:}
                   &= \sum_{i=1}^{n} \sum_{k=1}^{|\Lm(M)|} \text{lens}(M_i)(\iota_{k,i}) \\
                   \shortintertext{Holding $i$ constant for $\iota$ to separate each $M_i$:}
                   &= \sum_{i=1}^{n} \bigl(\sum_{k=1}^{|\Lm(M_i)|} \text{lens}(M_i)[k]\bigr)  * \prod_{j\in \{1, \ldots, n\}\setminus \{i\}}|\Lm(M_j)|\bigr) \\
                   &= \sum_{i=1}^{n}  \bigl(\norm{\Lm(M_i)} * \prod_{j\in \{1, \ldots, n\}\setminus \{i\}}|\Lm(M_j)| \bigr)\\
                   &< \bigl(\sum_{i=1}^{n}  \norm{\Lm(M_i)}\bigr) * |\Lm(M)| = |\Lm(M)| \sum_{i=1}^{n}  \norm{\Lm(M_i)} && \hspace{-.5em}\text{(Eq. 2)}
\end{align*}
\begin{itemize}
           \item[] For the size of $L_a$ it holds: 
\end{itemize}
       \begin{align*}
            \norm{L_a} &< 1 + |L_a|*\bigl( \norm{\Lm(M_1\pr)} + \norm{\Lm(M_n\pr)} \\ &\textcolor{white}{+} \hspace{6.17em}+ \sum_{r \in \{2, \ldots, n-1\} }^{n} \norm{\Lm(M_r)}\bigr) \\
            &< 1 + |L_a|\sum_{r \in \{1, \ldots, n\}}^{n} \norm{\Lm(M_r)}) < \norm{\Lm(M)} 
         \end{align*}
        \begin{itemize}
           \item[] because due to (Eq. 2) we can bound the size of $L_a$ by the number of traces times the sum of each children's log size and since $|L_a| < |\Lm(M)|$ proves the inequality even for the upper bound of unchanged $M_i$ and $M_j$ children in $M_a$.\\
        \end{itemize} 
        
        \noindent If $v$ is also in inverse strict order relation to some activities $q_1$ in $M_1$, then $M_1$ is split into two process trees $M_{1,\rightsquigarrow^{-1}}$ that contains the activities $q_1$ in inverse strict order relation to $v$ and $M_{1,\times}$ that contains the activities $q_1$ in choice relation to $v$. Analogously, if $u$ is also in strict order relation to some activities $q_n$ in $M_n$, then $M_n$ is split into two process trees $M_{n,\rightsquigarrow}$ that contains the activities $q_n$ in strict order relation to $u$ and $M_{n,\times}$ that contains the activities $q_n$ in choice relation to $u$. Taken together:
        \begin{align*}
        M_a = \:\rightarrow(M_{1,\rightsquigarrow^{-1}}, \times(x, \rightarrow(M_{1,\times}, M_{2}, \ldots, M_{n-1}, M_{n,\times}), M_{n,\rightsquigarrow}) && \text{(II)}
        \end{align*}
        The mdf-complete log $L_a$ is largest both in terms of number of traces and size, if $M_1$ and $M_n$ only contain either parallel $\wedge(\ldots)$ or sequence $\rightarrow(\ldots)$ nodes (cf. \autoref{alemma:sort}). It follows that all activities $q_1$ in $M_1$ are in inverse strict order relation to $v$, i.e., $v \rightsquigarrow^{-1} q_1$, and all activities $q_n$ in $M_n$ are in strict order relation to $v$, i.e., $v \rightsquigarrow q_n$. Hence, the $M_a$ with the largest $L_a$ for (II) is:
        \begin{align*}
        M_a = \:\rightarrow(M_{i}\pr, \times(x, \rightarrow(M_{2}, \ldots, M_{j-1}), M_{j}\pr) && \text{(III)}
        \end{align*}
        For (III), it follows that: 
        \begin{align*}
            |L_a| &= |\Lm(M_1\pr)| * |\Lm(M_n\pr)| * (1 + \prod_{r \in \{2, \ldots, n-1\} } |\Lm(M_r)|) \\
                            &= |\Lm(M_1\pr)| * |\Lm(M_n\pr)| && \text{(New)}\\
                             &\textcolor{white}{=}  \hspace{3em}+ |\Lm(M_1\pr)| * |\Lm(M_n\pr)| * \prod_{r \in \{2, \ldots, n-1\}} |\Lm(M_r)|) &&\text{(Abs.)} \\
                            &<  |\Lm(M)|
         \end{align*}
       \begin{itemize}
           \item[] because $M_1\pr$ and $M_n\pr$ both have parallel root nodes\footnote{$\disc_{IM}$ discovers process trees in \emph{normal form} \cite{leemans_robust_2022} such that for $M = \:\rightarrow(M_1, \ldots,M_n)$ no children of $M$ can have the sequence as a root node.} and the factorials in the multinomial coefficients of $|\Lm(M_1\pr)|$ and $ |\Lm(M_n\pr)|$ decrease ``faster'' through multiplication in (Abs.) than $|L_a|$ gains traces through adding $|\Lm(M_1\pr)| * |\Lm(M_n\pr)| $ in (New). For the size of the log $L_a$ it holds:
       \end{itemize}
       \begin{align*}
            \norm{L_a} &< |L_a| 
            * (1 + \norm{\Lm(M_1\pr)} + \norm{\Lm(M_n\pr)}   + \sum_{r \in \{2, \ldots, n-1\}}^{n} \norm{\Lm(M_r)}) \\
            &< \norm{\Lm(M)} 
         \end{align*}
        \begin{itemize}
           \item[] because we can bound the size of $L_a$ by the number of traces times the sum of each children's log size (Eq. 2) and two times a smaller log $\Lm(M_1\pr)$ and $\Lm(M_2\pr)$ outweighs the additional event.\\
        \end{itemize} 
    \begin{itemize}
        \item Case $\oplus = \:\wedge$: Analogous to case $\oplus = \times$.
    \end{itemize}

\end{proof} 

\setcounter{lemma}{7} 

\begin{lemma}[Matching of quotient sets]
\label{alemma:matching}
    If $L$ is restricted and $\ma_{bpa}$ applicable to $M = \disc_{IM}(L)$, there exists a matching between the quotient set $ L_{tmp}/\!\sim$ of event log $L_{tmp} = \ea^1_{bpa}(L,\disc_{IM},\ma_{bpa})$ and the quotient set $L_a/\!\sim$ of $L_a = \Lm(M_a) $ for $M_a = \ma_{bpa}(\disc_{IM}(L))$, i.e., $\operatorname{matching} : L_1 /\!\sim \;\rightarrow\; L_2 /\!\sim$ exists. Also, for every matched pair of equivalence classes $L_{tmp}^{class} \in L_{tmp}/\!\sim$ and $L_{a}^{class} = \operatorname{matching}(L_{tmp}^{class})$ it holds that $L_{tmp}^{class}$ has equal to or more traces than $L_{a}^{class}$: $ |L_{a}^{class}| \leq  | L_{tmp}^{class}|  $ (ii).
\end{lemma}
\begin{proof}
    We prove the statement in four steps (I-IV). First, we prove that $L_{tmp}$ and $L_a$ share the same activities $A_{L_{tmp}} = A_{L_a}$. Second, we prove that the quotient sets have the same size: $|\Ltr| = |\Lar|$. Third, we prove that for every equivalence class $\Ltc \in \Ltr$, there exists exactly one equivalence class $\Lac \in \Lar$ such that their traces are indistinguishable modulo transposition: $\forall \sigma_a \in \Ltc \sigma\in\Lac: \delta_{kendall}(\sigma_a, \sigma) \neq \bot$. 
    From the three steps I-III, the existence of a matching follows. Fourth, we prove the inequality (ii). 
    
    \noindent\textbf{I.} By code inspection of $\ea^1_{bpa}$ it follows that $A_{L_{tmp}} = A_{L_a}$, because the parameter $\agg$ is equally applied to $L$ for abstraction of concrete events into their abstract counterparts as it is applied to $M$ for abstraction of concrete activities into their abstract counterparts.

    \noindent\textbf{II.} Second, we establish that there are as many equivalence classes in $L_{tmp}/\!\sim$ as there are in $L_{a}/\!\sim$: $|\Ltr| = |\Lar|$. Towards contradiction, we assume $|\Ltr| \neq |\Lar|$. Given the model structure of $M_a$ (cf. \autoref{alemma:redisc2} and \autoref{def:cbpa}), the $\times$-node is the only process tree operator that affects the number of equivalence classes in $L_a$, because $\circlearrowleft$-nodes other than the self-loop as a leaf do not occur and traces in both the language of $\rightarrow$-node and $\wedge$-node are indistinguishable modulo transposition. Likewise, the restricted event log $L$ in \autoref{def:restricted} is similarly constrained. By code inspection of $\ea^1_{bpa}$ it follows that the resulting $L_{tmp}$ satisfies the requirements of \autoref{def:restricted}. Hence, $L_{tmp}$ is also a restricted event log. Let $M_{tmp} = \disc_{IM}(L_{tmp})$ be the process tree that the IM discovers from $L_{tmp}$. From $|\Ltr| \neq |\Lar|$, from the respective restrictions on both $L_{tmp}$ and $M_a$, and from \autoref{alg:minimal}, it follows that $M_{tmp}$ and $M_a$ must have a different number of $\times$-nodes. Hence, it follows that $|L_{tmp}| \neq |L_a|$, i.e., the two event logs have different numbers of traces (cf. \autoref{alemma:gensize} and line 8 \autoref{alg:minimal}). Considering that $|L_a| < \Lm(M)$ (cf. \autoref{lemma:large}) and that $\ea_{bpa}^1$ does not add or delete traces from $L$, it can only be that $|L_{tmp}| > |L_a|$. Consequently, the number of $\times$-nodes in $M_{tmp}$ must be larger than in $M_a$.
    
    From the last two statements, it follows that there exist two traces $\sigma_1, \sigma_2 \in L_{tmp}$ and two events in these traces $\exists x\in \sigma_1, y \in \sigma_2$ such that $x +_{M_{tmp}}y$ holds in the behavioral profile $p_{M_{tmp}}$ of $M_{tmp}$. Let $x$ and $y$ be activities that are children of one of the additional $\times$-nodes in $M_{tmp}$, i.e., $x \in A_{M_1}$ and $y \in A_{M_2}$ for two children $M_1$ and $M_2$ of the additional $\times$-node. From \textbf{1.}, it follows that $x,y \in A_{M_a}$. Additionally, $x \bigtriangleup_{M_a}y $ with $\bigtriangleup \in \{\rightsquigarrow, \rightsquigarrow^{-1}, \parallel\}$ in the behavioral profile $p_{M_a}$ of $M_a$, as otherwise $x$ and $y$ would not be in children of the additional $\times$-node in $M_{tmp}$. There are two alternative reasons for $x +_{M_{tmp}}y$. First, $x +_{M_{tmp}}y$ holds, because $\ea^1_{bpa}$ removed one of the two activities through the $\operatorname{deleteChoiceActivities}$ operator in line 16 \autoref{alg:ea1}. If $\operatorname{deleteChoiceActivities}$ does not remove the two activities, they must have been in choice relation already. In both cases, however, the activities $x$ and $y$ are in choice relation in the behavioral profile $p_{M_a}$ of $M_a$, i.e., a contradiction. Second, $x +_{M_{tmp}}y$ holds, because either $x$ or $y$ is a concrete activity $v$ such that one of the two was not added to an abstracted trace $\sigma_{abs}$ (cf. line 9 \autoref{alg:ea1}). Again, this can only happen, if the two activities are in choice relation in $p_{M_a}$. 

    \noindent\textbf{III.} Third, we establish that for every equivalence class $\Ltc \in \Ltr$, there exists exactly one equivalence class $\Lac \in \Lar$ such that: $\forall \sigma_{abs} \in \Ltc \sigma_a\in\Lac: \delta_{kendall}(\sigma_{abs}, \sigma_a) \neq \bot$. Towards contradiction, we assume that there exists an equivalence class $\Ltc \in \Ltr$ for which no equivalence class $\Lac \in \Lar$ exists that can be matched. Let $\Ltc$ be the equivalence class for which no matching equivalence class $\Lac \in \Lar$ exists. For every trace $\sigma_{abs} \in \Ltc$ and for every $\sigma_a \in \Lac$ it follows that $\delta_{kendall} = \bot$. Either the traces have all a different length $|\sigma_{abs}| \neq |\sigma_a|$ or have different activities $\operatorname{bag}(\sigma_{abs}) \neq \operatorname{bag}(\sigma_a)$ (cf. \autoref{def:equivalence}). 
    
    \noindent (Different lengths) By code inspection of $\ea_{bpa}^1$ it follows that a trace $\sigma_{abs}$ can only change its length relative to its concrete $\sigma$ by (1) deletion (line 8), (2) deletion (line 9), (3) insertion (line 12), or (4) deletion (line 17). The first deletion corresponds to abstraction of concrete events $\agg(x) \subseteq A_{\sigma}$ of activities occurring in $\sigma$ into their abstract events $x$. The second deletion corresponds to choice relations (cf. II). The third deletion corresponds to self-loops. The fourth deletion corresponds to choice relations (cf. II). As the same abstraction operation was applied to corresponding concrete activities in $M$ to yield abstract activities $x$ in $M_a$, both choice relations are similar for $L_a$ and $L_{tmp}$ (cf. II), and self-loops in $M_a$ corresponds to $\langle \ldots ,z, \ldots, z,\ldots \rangle \in L_a$ (cf. \autoref{alg:minimal}), the corresponding trace length of $\sigma_{abs}$ must occur for some trace $\sigma_a\pr$ in some equivalence class $\Lac \in \Lar$. 

    \noindent (Different multiset of activities) From (I), the different multisets of activities $\operatorname{bag}(\sigma_{abs}) \neq \operatorname{bag}(\sigma_a)$ cannot be due to activities that are only in one of the two event logs $L_{tmp}$ and $L_a$. Thus, either is an activity $x \in A_{L_{tmp}}$ only in one of the two multisets, i.e., $x \in \operatorname{bag}(\sigma_{abs}) $ and $x \not\in \operatorname{bag}(\sigma_a)$ without loss of generality, or the multiplicity of an activity $x$ is unequal in the two multisets, i.e, $\operatorname{bag}(\sigma_{abs})(x) \neq \operatorname{bag}(\sigma_a)(x)$. If an activity $x$ is only in one of the two multisets, the only reason can be a corresponding choice relation $x +_{M_a} y $ that prevented the $x$ activity to occur in trace $\sigma_{abs}$ due to $y$ occurring in $\sigma_{abs}$. However, from (II), the choice relation occurs also in $M_a$, and by correctness of \autoref{alg:minimal} (cf. \autoref{lemma:gensize}), also in $L_a$. Hence, the trace $\sigma_{abs}$ without the $x$ activity must have a matching trace $\sigma\pr \in \Lac$ with the same multiset of activities. If the multiplicity of an activity $x$ is unequal in the two multisets of activities, the only reason can, again, be a corresponding choice relation, as the loop is restricted to a self-loop both in $L$ and in $M_a$. Thus, there must be a matching trace $\sigma\pr \in \Lac$ with the same multiset of activities.

    \noindent Overall, it follows that for every equivalence class $\Ltc \in \Ltr$, there exists exactly one equivalence class $\Lac \in \Lar$ such that: $\forall \sigma_{abs} \in \Ltc \sigma_a\in\Lac: \delta_{kendall}(\sigma_{abs}, \sigma_a) \neq \bot$.

    \noindent\textbf{IV.} Lastly, we prove (ii). From \autoref{lemma:large}, it follows that $|L_a| < |L|$. Hence, $|L_a| < |L_{tmp}|$, i.e., the event log $L_a$ has fewer traces than the preliminary abstracted event log $L_{tmp}$, because $\ea_{bpa}^1$ does not delete traces from $L$. Consequently, the statement follows from (II), (III), $|L_a| < |L_{tmp}|$, and the minimality of $L_a$ (cf. \autoref{lemma:gensize}). 
\end{proof}

\setcounter{lemma}{1} 
 
\begin{lemma}[$\ea_{bpa}$ returns mdf-complete logs]
\label{alemma:df}
    If $L$ is restricted and $\ma_{bpa}$ applicable to $M = \disc_{IM}(L)$, the event log $L_a\pr = \ea_{bpa}(L,\disc_{IM},\ma_{bpa})$ is a mdf-complete event log for $M_a = \ma_{bpa}(\disc_{IM}(L)).$
\end{lemma}
\begin{proof}
    From line 1 of Algorithm 2 in our main paper it follows that $L_a$ is a mdf-complete event log for $M_a$ (cf. \autoref{lemma:gensize}). Hence, we must show $ L_a\pr = L_a$.
    
    \noindent The first step $\ea_{bpa}^1$ (Algorithm 1 in the main paper) abstracts the events $e \in \sigma$ whose concrete activities $e \in A_{\ma}$ are abstracted by $\ma_{bpa}$ into their respective new abstract activities $x \in A_{new} = A_{M_a}\setminus A_M$ (cf. condition 4 Definition 2 in the main paper). Hence, encountering an event $e \in \sigma$ (line 7) with $e \in A_{\ma}$ must trigger the construction of a new, abstract event $x$ (line 10). A new abstract activity $x \in A_{new}$ abstracts two or more concrete activities $\agg(x)$. Thus, all events $e\pr \in \sigma$ that have an activity to be abstracted by $x$, i.e., $e\pr \in \agg(x)$ must be abstracted into a single abstract event $x$ in trace $\sigma$. The condition in line 8 ensures that the first occurrence of an event $e$ to be abstracted into $x$ is the only $e$ that triggers the construction of $x$. Because $\agg(x)\cap \agg(y)\neq\emptyset$ with $x,y\in A_{new}$ is allowed, line 8 captures all new abstract activities $x,y, \ldots$ that abstract a concrete activity $v \in A_{\ma}$ equal to the event $e$. Since concrete activities $u \in A_{\neg\ma}$ must not be changed in $\sigma$, a choice relation between $u$ and an abstract activity $x$, i.e., $u +_{M_a} x \in p_{M_a}$, prohibits adding the corresponding abstract event $x$ to the preliminary abstracted $L_{tmp}$ (line 9). Since abstract activities can be in parallel relation to themselves $x \parallel_{M_a} x \in p_{M_a}$, the corresponding abstract event $x$ is added twice to inject the pattern for the self-loop (cf. Section 5.2 in the main paper). Lastly, events $e$ that are not abstracted, i.e., $ e \in A_{\neg \ma}$, are added to the abstracted trace $\sigma_{abs}$. 

    \noindent Lastly, $\ea_{bpa}^1$ computes a set of abstract activity sets $A_{\times}$ (line 16) such that for each two activities $x,y \in A, A \in A_{\times}$ it holds that $x +_{M_a} y \in p_{M_a}$. If a trace $\sigma_{abs} \in L_{tmp}$ contains at least two activities $x, y$ that are both in the same $A$ of $A_{\times}$, we must eliminate all but one of the activities $x,y$. The function $\operatorname{eliminateChoiceActivities}$ ensures elimination of the respective abstract activities per trace $\sigma_{abs}$ as specified in line 17. 

    \noindent The second step $\ea_{bpa}^2$ (Algorithm 2 in the main paper) identifies with $L_=$ all traces $\sigma \in L_a$ that are already contained in $L_{tmp}$ (line 2) and adds them to the abstracted event log $L_a\pr$ (line 3). While there are traces $\sigma  \in L_{open}$, $L_a\pr$ does not contain all required traces $\sigma \in L_a$. As the only difference between a trace $\sigma \in L_a$ and a trace $\sigma_{abs} \in L_{tmp} \setminus\, L_a\pr$ is the order of events, the minimal number of transpositions required to transform trace $\sigma_{abs}$ into trace $\sigma$ is computed by the \emph{Kendall Tau Sequence Distance} \cite{kendall_distance_2020} denoted by $\delta_{kendall}$. As the distance $\delta_{kendall}$ does not only compute the distance metric, but also the required transpositions, we apply the transpositions on $\sigma_{abs}$ as specified in line 7. 
    From \autoref{lemma:large}, it follows that $ |L_a|< |L_{tmp}|$ such that the while loop in line 5 always terminates. After termination of the while loop, it follows that $L_a\pr = L_a$.

    \noindent Overall, it follows that $L_a\pr = L_a$. Since $ |L_a|< |L_{tmp}|$ and $\norm{L_a} < \norm{L}$ (cf. \autoref{lemma:large}), the EA $\ea$ is well-defined.
\end{proof}

\setcounter{lemma}{10} 

\begin{lemma}[Restricted discovery]
\label{alemma:semid}
    If $L$ is restricted, $M = \disc_{IM}(L)$ is a process tree in $C_c$.
\end{lemma}
\begin{proof}
    IM does not discover duplicate activities (cf. Definition 5.7 property $C_B$.2 and Lemma 6.2 in \cite{leemans_robust_2022}), i.e., $M$ meets requirement 1 of $C_c$. Because IM does not discover a loop operator through a loop cut, does not execute fall through ``Flower Model'' as well as ``Tau Loop'', and executes fall through ``Strict Tau Loop'' only to discover a self-loop $\tloop(x, \tau)$ for $x \in \AL$, $M$ meets requirement 2 of $C_c$.
\end{proof}

\bibliographystyle{splncs04}

\end{document}